\documentclass{article}

     \PassOptionsToPackage{numbers, compress}{natbib}

\usepackage[dblblindworkshop, final]{neurips_2026}

\usepackage[utf8]{inputenc} 
\usepackage[T1]{fontenc}    
\usepackage{hyperref}       
\usepackage{url}            
\usepackage{booktabs}       
\usepackage{amsfonts}       
\usepackage{nicefrac}       
\usepackage{microtype}      
\usepackage{xcolor}         
\usepackage{comment}
\usepackage{float}
\usepackage{multirow}
\usepackage{microtype}

\usepackage{graphicx}
\usepackage{amsmath,amssymb,mathtools}
\usepackage{amsthm}
\usepackage{bbm} 
\usepackage{multirow}
\usepackage[table]{xcolor}
\usepackage{enumitem}
\usepackage{placeins}
\usepackage{subcaption}
\usepackage{pgf}
\usepackage{tikz}
\usepackage{adjustbox}
\usepackage{cleveref}
\usepackage{bbm} 
\usepackage{pgfplots} 
\usepackage{todonotes}
\usepackage{comment}
\usepackage[most]{tcolorbox}
\usepackage{algorithm}
\usepackage{algpseudocode}
\usepackage{paralist}
\usepackage{wrapfig}

\usepackage{thmtools,thm-restate}
\usepackage{svg}
\usepackage[most]{tcolorbox}
\usepackage{xcolor}
\definecolor{colfirst}{RGB}{0,114,178}
\definecolor{colsecond}{RGB}{213,94,0}
\usepackage{tabularx}
\usepackage{ragged2e}
\theoremstyle{plain}

\newlength{\figfourheight}
\definecolor{seabornGreen}{HTML}{029E72}
\definecolor{colfirst}{HTML}{029e72}
\definecolor{colsecond}{HTML}{de8f03}
\providecommand{\mathdefault}[1]{#1}
\newtcolorbox{desideratum}[1]{
    colback=seabornGreen!8,
    colframe=seabornGreen,
    boxrule=0.5pt,
    arc=2pt,
    left=6pt,
    right=6pt,
    top=5pt,
    bottom=5pt,
    title=\textbf{#1},
    fonttitle=\normalsize,
    coltitle=white
}

\workshoptitle{Linguistic Principles for Foundation Models (LP4FM)}

\title{
    What Is a Clarification? Estimating Model-Dependent Input Ambiguity
}

\author{%
  \textbf{Jingqin Tang$^{*1}$ \quad Tim Tomov$^{*123}$ \quad Sebastian Schmidt$^{1}$} \\
  \textbf{Dominik Fuchsgruber$^{12}$ \quad Stephan G{\"u}nnemann$^{123}$} \\[0.5em]
  $^1$ School of Computation, Information \& Technology, Technical University of Munich \\
  $^2$ Munich Data Science Institute \\
  $^3$ Munich Center for Machine Learning \\[0.5em]
  \texttt{\{jq.tang,tim.tomov,sebastian95.schmidt,d.fuchsgruber,s.guennemann\}@tum.de} \\
}

\begin{document}

\maketitle
\begingroup
\renewcommand{\thefootnote}{*}
\footnotetext{Equal contribution.}
\endgroup

\newcommand{\circleclar}{\textsc{CircleClar}}
\newcommand{\indicator}{\mathbbm{1}}

\begin{abstract}
\looseness=-1 Ambiguous inputs can admit multiple interpretations, each potentially yielding a different answer. Existing work largely treats this ambiguity as a property of the input, yet for a particular language model, only some alternatives may be known and answerable. We therefore distinguish world ambiguity, capturing all alternatives admitted by an input, from model ambiguity, capturing those accessible to the responding model. To characterize these alternatives, we formalize natural-language clarifications through three desiderata: \emph{uniqueness}, \emph{faithfulness}, and \emph{answer equivalence}. We show that these desiderata induce a one-to-one correspondence between valid answers and equivalence classes of clarifications, implying that any interpretation can be represented either by its answer or by the clarifications that uniquely express it. We leverage this correspondence to introduce \circleclar{}, a black-box method that estimates the set of answers available to a model by checking consistency along an answer--clarification--answer cycle. Across open-domain QA and executable Text-to-SQL, \circleclar{} achieves strong recovery of these answer sets. In QA, the recovered support yields the strongest overall performance in model-dependent ambiguity detection and enables effective selection among answering, clarifying, and abstaining. Together, these results demonstrate the practical value of a model-dependent view of ambiguity and its characterization through an answer-clarification correspondence.

\end{abstract}

\section{Introduction}
\label{sec:introduction}

Ambiguous inputs can admit multiple reasonable interpretations, each potentially leading to a different answer. Existing work has largely treated this ambiguity as a property of the input itself \citep{min-etal-2020-ambigqa,NEURIPS2024_a4c942a8,cole-etal-2023-selectively, pmlr-v235-hou24b,liu-etal-2023-afraid, kim-etal-2023-tree, saparina2026reasoningintentambiguousrequests}. However, when an ambiguous input is presented to a Large Language Model, not every valid interpretation is necessarily accessible and answerable by that model. Consider, for example, the question \emph{``Who voiced the snake in the Jungle Book?''}. The intended adaptation is unspecified: for the 1967 film, the answer is Sterling Holloway, whereas for the 2016 film, it is Scarlett Johansson. While both interpretations are valid, one model may know both films and their answers, whereas another may know only one. The ambiguity inherent in the question remains unchanged; what changes is \emph{how much of that ambiguity is accessible to the responding model}.

\looseness=-1 We therefore propose to distinguish between \emph{world ambiguity} and \emph{model ambiguity} (\Cref{fig:figure1}). World ambiguity captures the alternatives admitted by the input, independent of any particular model. Model ambiguity instead captures those alternatives that are \emph{known and answerable} to the responding model. The distinction is consequential: two models can face the same world ambiguity but, because they know different parts of its ambiguity, effectively face different problems. A model that knows several valid alternatives should recognize that committing to one leaves the user's intent unresolved; a model that knows only one alternative, or none, is in a fundamentally different situation. We therefore ask  \emph{how to estimate the interpretations and corresponding answers a particular model knows}.

To do so, we use natural-language clarifications \citep{min-etal-2020-ambigqa,kim-etal-2023-tree}, which are explicit representations of distinct interpretations of the original input. While such clarifications are widely used in prior work, their precise definition is often omitted \citep{pmlr-v235-hou24b,epub135635,rosseeuw2025clara}. Hence, we propose to characterize valid clarifications through three desiderata:
\begin{inparaenum}[(i)]
    \item \emph{uniqueness}: a clarification should admit only a single answer;
    \item \emph{faithfulness}: that answer must also be a valid answer for the original question; and
    \item \emph{answer equivalence}: clarifications that induce the same answer should be treated as equivalent.
\end{inparaenum}
Our key observation is that, under these desiderata, each valid answer corresponds to exactly one clarification class, thereby providing two complementary representations of the same interpretation (\Cref{fig:figure1}b). 

This correspondence suggests a simple way to estimate what ambiguity a model can realize. If a model proposes an answer, expresses the corresponding interpretation through a clarification, and then reliably returns to the same answer under that clarification, the resulting cycle provides additional evidence that this alternative is accessible to the model. We turn this idea into \circleclar{}, a black-box method for estimating the model-specific set of realizable answers (\Cref{fig:figure1}c). We evaluate \circleclar{} on open-domain QA using AmbigQA \citep{min-etal-2020-ambigqa} and our curated high-quality version (AmbigQA-HQ), and on executable Text-to-SQL using AMBROSIA \citep{NEURIPS2024_a4c942a8}. \circleclar{} achieves the best answer-set F1 in seven of twelve model–dataset settings and outperforms direct-generation baselines across all QA settings. Moreover, its cardinality can predict model-dependent ambiguity and serves as an effective basis for deciding whether to answer, clarify, or abstain. Taken together, these results support a model-dependent view of ambiguity: what matters for a particular LLM is not only which alternatives an input admits, but which of them the model can realize.

\begin{figure}[t]
    \centering
    \scalebox{1}[1.08]{%
        \includegraphics[width=.99\linewidth]{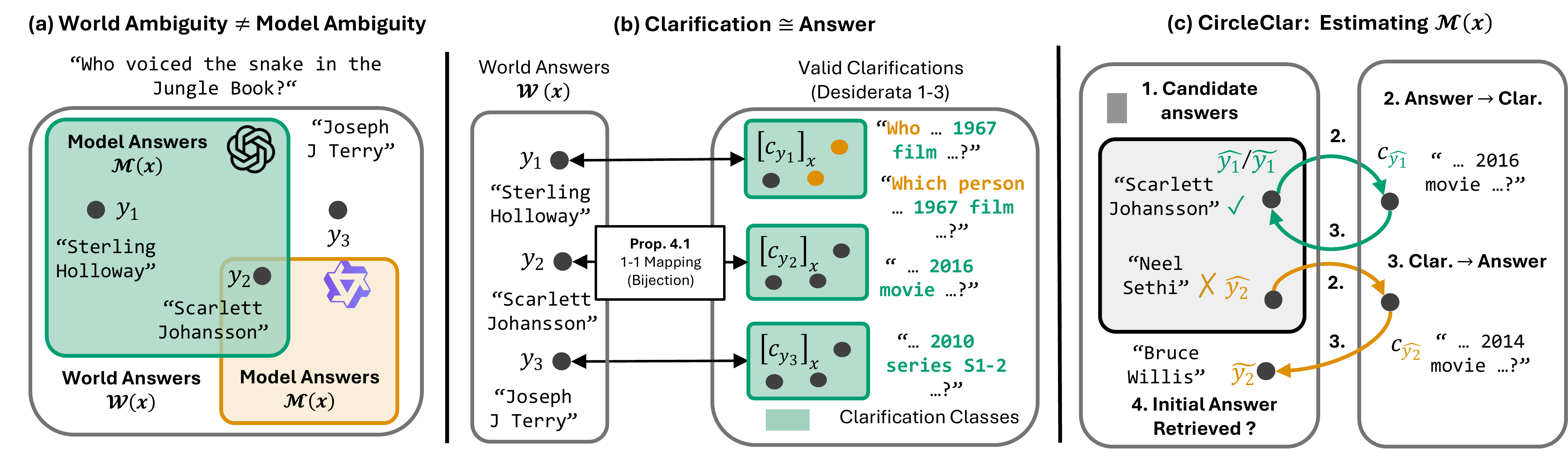}
    }
    \vspace{-0.25em}
    \caption{
        \textbf{Overview of our framework}: illustrated with the ambiguous question ``Who voiced the snake in \textit{The Jungle Book}?''.
        (a) The world answer support \(\mathcal{W}_T(x)\) is fixed by the input, while the model-realizable support depends on the model's knowledge.
        (b) Proposition~4.1 establishes a one-to-one correspondence between world-supported answers and clarification classes representing the same interpretations.
        (c) \circleclar{} estimates the model-realizable support by testing whether candidate answers are recovered through answer--clarification--answer cycles.
    }
    \label{fig:figure1}
    \vspace{-0.35em}
\end{figure}

\section{Related Work}

\paragraph{Representing input ambiguity.}
\looseness=-1 A common way to study input ambiguity is to make the possible interpretations of an underspecified input explicit.
In open-domain QA, interpretations can be expressed through disambiguated questions, either provided as annotations \citep{min-etal-2020-ambigqa} or generated recursively as part of answering ambiguous questions \citep{kim-etal-2023-tree}.
In semantic parsing, where natural-language inputs are mapped to executable representations, such as SQL queries or tool calls, different interpretations can be expressed as distinct executable programs~\citep{bhaskar-etal-2023-benchmarking,NEURIPS2024_a4c942a8}, or first stated in natural language before being mapped to logical forms~\citep{saparina-lapata-2025-disambiguate}. Related work also makes alternative readings explicit as structured interpretation--answer pairs~\citep{saparina2026reasoningintentambiguousrequests}, while \citet{tomov2026illusioncertaintyuncertaintyquantification} take a distributional view, representing ambiguity through a ground-truth distribution over semantic answer classes and their relative likelihoods.
Beyond such explicit representations, \citet{pmlr-v267-kirchhof25b} takes a more general view of underspecification, describing it as uncertainty about the intended task or as missing contextual information, modeled via a latent variable model.

\paragraph{Estimating input ambiguity.} 
Recent work treats input ambiguity as a quantity that can be estimated from model outputs.
ICE~\citep{pmlr-v235-hou24b} treats it as a source of aleatoric uncertainty. It generates multiple clarifications of the input, obtains predictions for each, and ensembles them to separate aleatoric uncertainty from epistemic uncertainty.
Spectral Uncertainty~\citep{epub135635} extends this approach by doing soft clustering via embeddings.
CLARA~\citep{rosseeuw2025clara} generates multiple clarifications and uses their semantic diversity as an ambiguity signal, while ~\citet{shi-etal-2025-ambiguity} detect ambiguity from sampled-answer distributions.
CLUES~\citep{ziletti-dambrosi-2026-disentangling} instead uses a two-stage process from interpretations to answers, separating uncertainty across interpretations from instability in their associated answers.

\paragraph{Clarifications \& input ambiguity.}
Beyond estimating ambiguity itself, other work incorporates ambiguity detection or recognition into clarification systems, benchmarks, and policies.
~\citet{lee-etal-2023-asking} formulate open-domain QA as a pipeline that detects ambiguous questions, presents candidate interpretations in a clarification question, and answers based on the user's response.
CLAM~\citep{kuhn2023clamselectiveclarificationambiguous} similarly classifies whether a question is ambiguous and selectively requests clarification before answering.
CLAMBER~\citep{zhang-etal-2024-clamber} jointly benchmarks the identification of ambiguous information needs and the generation of clarification questions.
INTENT-SIM~\citep{zhang-choi-2025-clarify} estimates entropy over simulated user intents to determine when clarification is beneficial.
ACT~\citep{chen2025learning} learns multi-turn clarification policies and evaluates implicit ambiguity recognition through models' clarify-or-answer decisions.
Zhang, Knox, and Choi~\citep{zhang2025modeling} use simulated future conversation turns to train clarify-or-answer behavior and evaluate the resulting decisions against ambiguity labels.
Su and Cardie~\citep{su2026knowingshowingllmsrecognize} further show that ambiguity recognition capability does not necessarily translate into clarification behavior: models often identify an input as ambiguous when asked explicitly, yet default to answering in ordinary QA.
Closely related to our work, BAG~\citep{baan2026clarify} uses sampled responses as a belief state to guide whether to answer, clarify, or abstain. Its BAG3 variant groups responses by meaning and proposes interpretations for each group to distinguish alternative readings from factual inconsistencies.

\looseness=-1Overall, prior work studied how ambiguity can be represented, estimated, and handled through clarifications. Far less attention has been paid to how ambiguity varies with the responding model's own knowledge. APA~\citep{kim-etal-2024-aligning} recognizes this, defining perceived ambiguity as a property of a particular model and measuring it through the reduction in the model's average token entropy from the original question to a self-disambiguated version.
While conceptually sharing this perspective, our work goes further by formally distinguishing the two different types of ambiguity and defining not only a quantitative estimator but also actionable estimates of clarifications and their answers.




\section{Why Input Ambiguity Should be Model Dependent}
\label{sec:motivation}

\looseness=-1 Ambiguity is typically treated as a property of the input: a question is ambiguous when it admits multiple valid interpretations. As such, most prior work \citep{min-etal-2020-ambigqa,cole-etal-2023-selectively,liu-etal-2023-afraid,kim-etal-2023-tree,pmlr-v235-hou24b,NEURIPS2024_a4c942a8} studies ambiguity independently of the used model. Similar to \citet{kim-etal-2024-aligning}, we propose to \textit{treat ambiguity as dependent on model knowledge and, therefore, as a model-dependent property}. We formalize this distinction below and empirically show that it varies substantially across models (\Cref{fig:model-realizable-space}), motivating our model-dependent framework for characterizing ambiguity.

\paragraph{Preliminaries.}
LLMs implicitly define a probability distribution \(p(s)\) over arbitrary-length strings \(s \in \mathcal{V}^{<\infty}\) from a vocabulary \(\mathcal{V}\). Conditioned on an input \(x \in \mathcal{V}^{<\infty}\), the model induces a task-specific predictive distribution \(p_T(y \mid x)\) over answers \(y\) for a task \(T\). We analogously denote by \(p_T^\star(y \mid x)\) the corresponding unknown ground-truth distribution. Importantly, we consider answers in a task-relevant output space \(\mathcal{Y}_T\), which abstracts away distinctions that are irrelevant to the task, like surface form or syntax, by aggregating responses with similar semantic meaning into one element \cite{tomov2026taskawareness}. For question answering, for example, semantically equivalent answer strings correspond to the same element of \(\mathcal{Y}_T\); for executable outputs, such as SQL queries, syntactically different forms are treated as equivalent when they return the same requested information.

\textbf{World vs. Model Ambiguity.}
An ambiguous question \(x\) admits multiple interpretations, which may in turn lead to different correct answers. 
In the running example, the unspecified film yields Sterling Holloway or Scarlett Johansson. 
We call the set of all such answers the \emph{world answer support},
\[
    \mathcal{W}_T(x)
    :=
    \operatorname{supp} p_T^\star(\cdot \mid x)
    =
    \{y \in \mathcal{Y}_T : p_T^\star(y \mid x) > 0\},
\]
where \(p_T^\star\) denotes an oracle or ground-truth mapping.
The corresponding \emph{world ambiguity} is simply
\(|\mathcal{W}_T(x)|\), the number of correct answers admitted by the question.
In particular, \(x\) is \emph{world} ambiguous whenever there is more than one correct answer, i.e.
\(|\mathcal{W}_T(x)| > 1\).

Crucially, however, world ambiguity need not coincide with the ambiguity realized by a particular model.
Returning to our example, a model that knows both Jungle Book films can realize that the question admits at least two interpretations and associate each with its corresponding answer. In contrast, another model may only know the 2016 film---due, for instance, to differences in training data, capacity, distillation, or compression---and therefore realize only Scarlett Johansson as a valid answer.
The underlying question is equally ambiguous in both cases: its world answer support \(\mathcal{W}_T(x)\) has not changed.
What changes is which parts of this ambiguity are accessible to the model. To capture this distinction, for a model \(p_T\) we define its \emph{model-realizable answer support} as
\[
    \mathcal{M}_T(x)
    :=
    \operatorname{supp} p_T(\cdot \mid x) \cap \mathcal{W}_T(x)
    =
    \{y \in \mathcal{W}_T(x) : p_T(y \mid x) > 0\}.\footnote{Technically, an LLM without top-\(p\) or top-\(k\) sampling assigns non-zero probability to every sequence; in practice, we consider outcomes with non-negligible probability mass.}
\]
\begin{wrapfigure}{r}{0.37\columnwidth}
    \centering
    \resizebox{\linewidth}{!}{%
        \input{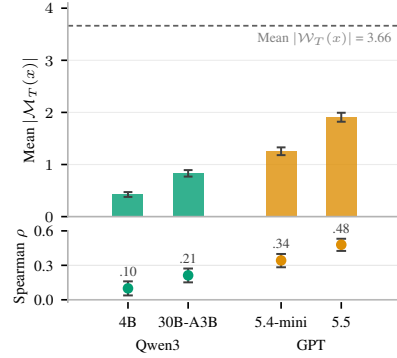}%
    }
    \caption{\textbf{Model-realizable support on AmbigQA.} Mean size (top) and Spearman correlation with world support size (bottom). Error bars show 95\% bootstrap CIs over questions.}
    \label{fig:model-realizable-space}
    \vspace{-2em}
\end{wrapfigure}

Accordingly, we define the \emph{model-realizable ambiguity} of \(x\) as
\(|\mathcal{M}_T(x)|\).
Unlike world ambiguity, this quantity is inherently model-dependent:
Two models facing the same question may realize different subsets of the valid answers.
We observe precisely this effect in \Cref{fig:model-realizable-space}: larger, more capable models tend to realize more of the ambiguity present in the world. Their model-realizable support also tracks the world support more closely across inputs, as shown by a higher input-level Spearman correlation between
$\lvert\mathcal{M}_T(x)\rvert$ and $\lvert\mathcal{W}_T(x)\rvert$. Nonetheless, even the strongest model remains substantially below the mean world-support size.

This distinction between these two forms of ambiguity matters because, when deciding how a \emph{particular model} should respond, world ambiguity alone is insufficient.
Suppose \(|\mathcal{W}_T(x)|>1\), but a model realizes only one of these answers.
\textit{From an oracle perspective, the question is ambiguous. From the model's perspective, there is no alternative interpretation on which it can meaningfully act.}
Conversely, a model that realizes several valid answers should recognize that answering with any single one implicitly commits to a specific interpretation.
Thus, the appropriate response---whether to answer, clarify, or abstain---depends not only on which answers exist, but on which of them are realizable by the model.
Our goal is therefore not to recover \(\mathcal{W}_T(x)\), but to estimate the model-specific set \(\mathcal{M}_T(x)\).

\section{A Clarification Framework for Estimating Model-Dependent Ambiguity}

\looseness=-1 We now turn to the question of how to estimate the model-realizable answer support \(\mathcal{M}_T(x)\). To do so, we use clarifications, which provide a natural way to express interpretations as natural-language inputs for the model.
While clarifications have been previously used \citep{min-etal-2020-ambigqa,
kim-etal-2023-tree,
pmlr-v235-hou24b,
rosseeuw2025clara},  a precise formal characterization of what constitutes a valid clarification has, to our knowledge, not been provided. This absence is consequential: without a rigorous definition, clarifications may possess undesirable properties that introduce confounds into downstream estimators. To address this gap, we are the first to formally characterize valid clarifications through a set of explicit desiderata (\Cref{sec:4.1}). Crucially, we prove that, under these desiderata, the set of clarifications admits an \emph{equivalent representation} of the answer space. We then exploit this equivalence to construct \circleclar{}, a black-box estimator of $\mathcal{M}_T(x)$.



\subsection{Formally Defining Clarifications} 
\label{sec:4.1}
We begin with the most basic requirement: a clarification must actually resolve the ambiguity. Consider again the ambiguous question \emph{``Who voiced the snake in The Jungle Book?''} in which the adaptation is unspecified. A logical clarification would be \emph{``Who voiced the snake in the 1967 animated film \textit{The Jungle Book}?''}. This reformulation resolves the ambiguity by selecting a single interpretation and, consequently, a particular answer. This motivates our first desideratum.

\begin{desideratum}{Desideratum 1: Uniqueness} A clarification shall uniquely determine an answer: \[ \exists y_c \in \mathcal{Y}_T \quad \text{such that} \quad p_T^\star(y_c \mid c)=1. \] \end{desideratum} Without uniqueness, the clarification itself remains ambiguous and therefore does not isolate one particular unambiguous reading of the original input. Uniqueness alone, however, is not sufficient. For example, \emph{``Who voiced Mowgli in the 1967 animated film \textit{The Jungle Book}?''} is also unambiguous, but its answer is not a valid answer to the original question about the snake. A clarification must therefore not only resolve the ambiguity, but also preserve the meaning of the original input.

\begin{desideratum}{Desideratum 2: Faithfulness} The answer induced by a clarification must also be supported by the original ambiguous input prompt: \[ y_c \in \mathcal{W}_T(x). \] \end{desideratum} 

\looseness=-1 Together, \emph{Uniqueness} and \emph{Faithfulness} ensure that a clarification selects exactly one valid answer from \(\mathcal{W}_T(x)\). In our example, specifying the 1967 film selects Sterling Holloway, while specifying the 2016 film selects Scarlett Johansson; both are valid answers to the original question. Importantly, however, the clarification itself is not unique. The same interpretation can typically be expressed in different ways. For instance, \emph{``Who voiced Kaa in Disney's 1967 animated \textit{The Jungle Book}?''} and \emph{``Who is the voice actor for the snake Kaa in the 1967 \textit{Jungle Book} film?''} differ in wording, but both select the same answer, Sterling Holloway. For our purposes, these surface-form differences are irrelevant.

\begin{desideratum}{Desideratum 3: Answer Equivalence} Let $\mathcal{C}_T(x)$ denote the set of valid clarifications of $x$ for task $T$. Two clarifications are equivalent exactly when they induce the same answer. For two valid clarifications \(c,c'\in \mathcal{C}_T(x)\), \[ c \sim_x c' \quad\Longleftrightarrow\quad y_c = y_{c'}. \] \end{desideratum}

\emph{Answer Equivalence} therefore groups different surface forms that express the same interpretation \textit{into classes}. In our example, all valid clarifications that specify the 1967 film belong to one class, since they all induce the answer Sterling Holloway. We denote the class containing a clarification \(c\) by
\(
    [c]_x
    :=
    \{c' \in \mathcal{C}_T(x) : c' \sim_x c\},
\)
and write \(y_{[c]}\) for the answer shared by all clarifications in this class, while writing \(\mathcal{C}_T(x)/{\sim_x}\) for the set of all such equivalence classes.

The three desiderata give us a useful structure: Each clarification class identifies exactly one valid answer to the original input. If we additionally assume that every valid answer can be expressed by at least one clarification, then the reverse also holds: every answer in \(\mathcal{W}_T(x)\) is represented by exactly one clarification class. This gives the following correspondence (Proof in Appendix~\ref{app:proof-bijection}):

\begin{restatable}[Clarification--Answer Bijection]
    {proposition}{clarificationbijection}
\label{prop:clarification-bijection}
Suppose clarifications satisfy D1--D3 and every
\(y \in \mathcal{W}_T(x)\) admits at least one valid clarification.
Then the map from clarification to response
\[
\Phi_x : \mathcal{C}_T(x)/{\sim_x}
\longrightarrow \mathcal{W}_T(x),
\qquad
\Phi_x([c]_x) = y_{[c]},
\]
is a bijection.
\end{restatable}

For each \(y \in \mathcal{W}_T(x)\), we write \([c_y]_x\) for the unique clarification class corresponding to \(y\) under this bijection. Returning once more to our example, the interpretation referring to the 1967 film can now be represented either by the answer \emph{Sterling Holloway} or by the entire class of clarifications that uniquely lead to this answer, such as \emph{``Who voiced the snake in the 1967 animated film The Jungle Book?''}. More generally,
Proposition~\ref{prop:clarification-bijection} gives us two equivalent
representations of each interpretation: its answer \(y\) and a class
\([c_y]_x\) of natural-language clarifications that induce it. Hence, starting from an answer \(y\), we can move to its corresponding clarification class \([c_y]_x\), and any clarification from this class maps back to the same answer \(y\). This yields the cycle \(y \longrightarrow [c_y]_x \longrightarrow y\), the key consistency relation underlying the estimator introduced next.

\begin{algorithm}[t!]
\caption{CircleClar}
\label{alg:mros}
\begin{algorithmic}[1]
\Require Input $x$, model $M$, proposal budget $B$, recovery samples $R$,
stability threshold $\tau$
\Ensure Estimated model-realizable answer support $\widehat{\mathcal{M}}_T(x)$

\State $\widehat{\mathcal{M}}_T(x) \gets \emptyset$
\State Draw $B$ answer collections $Y_1,\ldots,Y_B$ from $M$ on $x$
\State $\widetilde{\mathcal{Y}}_T(x) \gets
\operatorname{Canon}_T\!\left(\bigcup_{b=1}^{B}Y_b\right)$

\For{$\widetilde{y} \in \widetilde{\mathcal{Y}}_T(x)$}
    \State Generate a clarification $c_{\widetilde{y}}$ targeting
    $\widetilde{y}$
    \State Re-answer $c_{\widetilde{y}}$ independently $R$ times and map
    the responses to semantic answer classes $z_1,\ldots,z_R$
    \State $\widehat{q}_{\widetilde{y}} \gets
    \frac{1}{R}\sum_{r=1}^{R}\indicator\{z_r=\widetilde{y}\}$
    \If{$\widetilde{y}$ is the unique most frequent answer and
    $\widehat{q}_{\widetilde{y}}\geq\tau$}
        \State $\widehat{\mathcal{M}}_T(x)
        \gets \widehat{\mathcal{M}}_T(x) \cup \{\widetilde{y}\}$
    \EndIf
\EndFor

\State \Return $\widehat{\mathcal{M}}_T(x)$
\end{algorithmic}
\end{algorithm}

\subsection{\circleclar{}: Estimating Model-Realizable Answer Support}
\label{sec:circleclar}



The answer--clarification correspondence established under the ideal task distribution \(p_T^\star\) provides a consistency criterion for testing candidate answers using the model distribution \(p_T\):
if a candidate \(\widetilde y\) corresponds to an answer of a valid interpretation of \(x\), then expressing that interpretation through a clarification should make \(\widetilde y\) uniquely correct, and the model should recover it when answering the clarified input. Reliably recovering \(\widetilde y\) under \(p_T\) therefore provides evidence that \(\widetilde y \in \mathcal{M}_T(x)\).

\looseness=-1 As illustrated in Figure~\ref{fig:figure1}c, \circleclar{} operationalizes this idea in two steps: it first constructs a set of candidate outcomes supported by the model, and then uses the answer--clarification--answer cycle as a consistency test to filter these candidates. Concretely, we first query the model $B$ times with an enumeration prompt that asks for all plausible answers to the original input $x$, obtaining answer collections $Y_1,\ldots,Y_B$. We then take the union over these $B$ answer sets and map the resulting responses to semantic answer classes, yielding the candidate set \(\widetilde{\mathcal{Y}}_T(x)\). This candidate set provides an approximation to the model's predictive support, \(\widetilde{\mathcal{Y}}_T(x) \approx \operatorname{supp} p_T(\cdot \mid x)\). For each candidate \(\widetilde{y} \in \widetilde{\mathcal{Y}}_T(x)\), we ask the model to generate a clarification \(c_{\widetilde{y}}\) that makes the corresponding interpretation explicit. We then close the cycle by prompting the same LLM to answer this clarification again. If the clarification truly represents the same interpretation, the model should return to the original candidate \(\widetilde{y}\). We capture this through the recovery probability \(p_T(\widetilde{y}\mid c_{\widetilde{y}})\), retaining candidates whose recovery probability meets or exceeds a threshold \(\tau\). In practice, we approximate \(p_T(\widetilde{y}\mid c_{\widetilde{y}})\) using \(R\) Monte Carlo samples from \(p_T(\cdot\mid c_{\widetilde{y}})\). The complete algorithm is in \Cref{alg:mros} with prompts in Appendix~\ref{app:implementation}.

\section{Experiments}
\label{sec:evaluation}

Our primary evaluation asks how accurately \circleclar{} recovers the model-realizable answer support.
We then evaluate two downstream uses of the recovered support in QA: using its cardinality to detect model-dependent ambiguity and deriving the response action from the estimated support.

\paragraph{Datasets.}
We evaluate \circleclar{} across two domains. For open-domain question answering, we use 1{,}000 ambiguous questions from AmbigQA \citep{min-etal-2020-ambigqa}, together with AmbigQA-HQ, a manually curated set of 100 questions derived primarily from AmbigQA with more complete and reliable rewrite--answer annotations (see Appendix~\ref{app:ambigqa-hq}). For executable Text-to-SQL, we evaluate on 96 ambiguous requests from AMBROSIA \citep{NEURIPS2024_a4c942a8}. Since we treat final answers \(y\) by their task-relevant semantic meaning rather than as raw strings, we construct the corresponding equivalence classes using an LLM-based semantic grouping procedure tailored to the task (Appendix~\ref{app:implementation}).

\paragraph{Model-Realizable Reference Support.}
We construct $\mathcal{W}_T(x)$ from the distinct reference answers associated with the dataset-provided resolved inputs. To construct the model-specific reference $\mathcal{M}_T(x)$, we test whether the evaluated model reliably recovers each reference answer $y$ from its associated ground-truth clarification $c_y$ and if so, include it. (See Appendix~\ref{app:implementation} for more details).

\begin{table}[t]
    \centering
    \caption{\textbf{Model-realizable answer support recovery on QA and Text-to-SQL.}
    We report mean per-question set $F_1 \pm$ standard deviation and exact set match (EM) (\textcolor{colfirst}{Best} and \textcolor{colsecond}{second-best}).}
    \label{tab:qa_set_recovery}
    \vspace{3pt}
    \setlength{\tabcolsep}{2pt}
    \resizebox{1\columnwidth}{!}{%

\definecolor{colfirst}{HTML}{029E73}   
\definecolor{colsecond}{HTML}{DE8F05}  

\newcommand{\fem}[3]{%
    {$#1$$\scriptscriptstyle{\pm #2}$}\,/\,$#3$%
}

\begin{tabular}{ll|cccc}
\toprule

\multirow{2}{*}{\textbf{Dataset}} &
\multirow{2}{*}{\textbf{Method}} &
\textbf{Qwen3-4B} &
\textbf{Qwen3-30B-A3B} &
\textbf{GPT-5.4-mini} &
\textbf{GPT-5.5} \\

& &
$F_1$ / EM &
$F_1$ / EM &
$F_1$ / EM &
$F_1$ / EM \\
\midrule

\multirow{4}{*}{\shortstack{\small AmbigQA\\\small HQ}} & \small{ICE}
& \fem{\textcolor{colsecond}{\mathbf{39.23}}}{46.09}{\textcolor{colsecond}{\mathbf{33.00}}}
& \fem{\textcolor{colsecond}{\mathbf{52.52}}}{41.97}{\textcolor{colsecond}{\mathbf{33.00}}}
& \fem{\textcolor{colsecond}{\mathbf{63.93}}}{34.35}{\textcolor{colsecond}{\mathbf{31.00}}}
& \fem{\textcolor{colfirst}{\mathbf{77.91}}}{29.61}{\textcolor{colfirst}{\mathbf{50.00}}}
\\

 & \small{Sampling-50}
& \fem{33.07}{43.78}{25.00}
& \fem{43.63}{40.11}{22.00}
& \fem{55.51}{35.23}{24.00}
& \fem{61.79}{27.13}{21.00}
\\

 & \small{1-Enumeration}
& \fem{23.42}{38.45}{15.00}
& \fem{41.38}{36.37}{12.00}
& \fem{59.40}{34.74}{28.00}
& \fem{68.48}{29.52}{\textcolor{colsecond}{\mathbf{35.00}}}
\\

 & \cellcolor[gray]{0.9}{\small{\circleclar{}}}
& \cellcolor[gray]{0.9}{\fem{\textcolor{colfirst}{\mathbf{60.02}}}{46.26}{\textcolor{colfirst}{\mathbf{52.00}}}}
& \cellcolor[gray]{0.9}{\fem{\textcolor{colfirst}{\mathbf{64.88}}}{40.76}{\textcolor{colfirst}{\mathbf{49.00}}}}
& \cellcolor[gray]{0.9}{\fem{\textcolor{colfirst}{\mathbf{74.09}}}{34.78}{\textcolor{colfirst}{\mathbf{54.00}}}}
& \cellcolor[gray]{0.9}{\fem{\textcolor{colsecond}{\mathbf{68.85}}}{31.59}{\textcolor{colsecond}{\mathbf{35.00}}}}
\\

\midrule

\multirow{4}{*}{\shortstack{\small AmbigQA}} & \small{ICE}
& \fem{\textcolor{colsecond}{\mathbf{33.71}}}{44.54}{\textcolor{colsecond}{\mathbf{26.90}}}
& \fem{\textcolor{colsecond}{\mathbf{41.04}}}{42.63}{\textcolor{colsecond}{\mathbf{25.50}}}
& \fem{\textcolor{colfirst}{\mathbf{55.18}}}{42.50}{\textcolor{colfirst}{\mathbf{35.50}}}
& \fem{\textcolor{colfirst}{\mathbf{63.93}}}{35.82}{\textcolor{colfirst}{\mathbf{31.70}}}
\\

 & \small{Sampling-50}
& \fem{24.85}{40.58}{19.60}
& \fem{31.66}{42.25}{22.30}
& \fem{39.56}{40.10}{20.60}
& \fem{44.84}{36.69}{17.80}
\\

 & \small{1-Enumeration}
& \fem{14.33}{32.46}{10.20}
& \fem{24.50}{33.90}{7.40}
& \fem{38.28}{36.54}{12.30}
& \fem{46.93}{35.57}{16.80}
\\

 & \cellcolor[gray]{0.9}{\small{\circleclar{}}}
& \cellcolor[gray]{0.9}{\fem{\textcolor{colfirst}{\mathbf{48.28}}}{48.53}{\textcolor{colfirst}{\mathbf{44.70}}}}
& \cellcolor[gray]{0.9}{\fem{\textcolor{colfirst}{\mathbf{46.26}}}{45.28}{\textcolor{colfirst}{\mathbf{35.00}}}}
& \cellcolor[gray]{0.9}{\fem{\textcolor{colsecond}{\mathbf{46.55}}}{41.90}{\textcolor{colsecond}{\mathbf{27.60}}}}
& \cellcolor[gray]{0.9}{\fem{\textcolor{colsecond}{\mathbf{49.32}}}{36.49}{\textcolor{colsecond}{\mathbf{19.90}}}}
\\

\midrule

\multirow{4}{*}{\shortstack{\small AMBROSIA}} & \small{ICE}
& \fem{\textcolor{colfirst}{\mathbf{44.44}}}{37.49}{\textcolor{colsecond}{\mathbf{17.71}}}
& \fem{36.41}{33.73}{8.33}
& \fem{41.84}{32.92}{8.33}
& \fem{48.02}{29.95}{8.33}
\\

 & \small{Sampling-50}
& \fem{34.44}{40.46}{\textcolor{colsecond}{\mathbf{17.71}}}
& \fem{39.50}{36.61}{11.46}
& \fem{36.13}{31.34}{5.21}
& \fem{45.51}{31.33}{7.29}
\\

 & \small{1-Enumeration}
& \fem{38.70}{34.62}{11.46}
& \fem{\textcolor{colsecond}{\mathbf{45.45}}}{36.91}{\textcolor{colsecond}{\mathbf{18.75}}}
& \fem{\textcolor{colsecond}{\mathbf{53.26}}}{37.99}{\textcolor{colfirst}{\mathbf{28.12}}}
& \fem{\textcolor{colfirst}{\mathbf{51.98}}}{32.31}{\textcolor{colsecond}{\mathbf{11.46}}}
\\

 & \cellcolor[gray]{0.9}{\small{\circleclar{}}}
& \cellcolor[gray]{0.9}{\fem{\textcolor{colsecond}{\mathbf{42.64}}}{40.67}{\textcolor{colfirst}{\mathbf{21.88}}}}
& \cellcolor[gray]{0.9}{\fem{\textcolor{colfirst}{\mathbf{51.39}}}{39.76}{\textcolor{colfirst}{\mathbf{25.00}}}}
& \cellcolor[gray]{0.9}{\fem{\textcolor{colfirst}{\mathbf{56.39}}}{37.84}{\textcolor{colsecond}{\mathbf{26.04}}}}
& \cellcolor[gray]{0.9}{\fem{\textcolor{colsecond}{\mathbf{48.63}}}{34.53}{\textcolor{colfirst}{\mathbf{12.50}}}}
\\

\bottomrule
\end{tabular}%

    }
    \vspace{-1em}
\end{table}

\subsection{Estimating the Model-Realizable Answer Support $\mathcal{M}_T(x)$}
\label{sec:set-recovery}


\paragraph{Baselines.}
Existing methods cannot directly recover a model-realizable answer set. We therefore adapt ICE~\citep{pmlr-v235-hou24b} to recover answer sets by retaining answers with probability above a threshold, which is optimized per model–dataset pair to maximize the mean set F1. We further compare against two direct-generation baselines: \textsc{Sampling-50}, which samples 50 times, and \textsc{1-Enumeration}, which asks the model once to enumerate all plausible answers. Full details are provided in Appendix~\ref{app:implementation}.


\paragraph{Results.}
Table~\ref{tab:qa_set_recovery} shows that \circleclar{} performs strongly across datasets, achieving particularly large gains over baselines on smaller models and on the curated AmbigQA-HQ benchmark. The overall performance gap between AmbigQA and AmbigQA-HQ likely stems from annotation errors in the standard AmbigQA, which introduce noisy estimates of the ground-truth set. Across all methods, performance gaps shrink on larger models. Notably, the simple baseline of enumerating possible answers one-shot achieves remarkably high performance on larger models, indicating that increased scale enables models to surface multiple plausible alternatives directly within a single structured response.
\begin{wrapfigure}{r}{0.45\columnwidth}
    \centering
    \resizebox{\linewidth}{!}{%
        \input{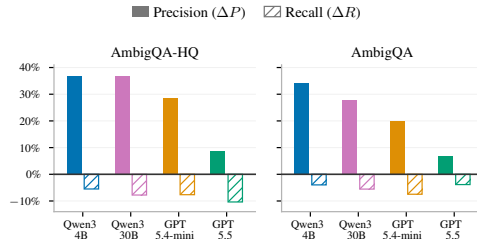}%
    }
    \caption{
        \textbf{Effect of consistency check on support recovery.}
        Change in macro precision and recall from the initial candidate set
        \(\widetilde{\mathcal{Y}}_T(x)\) to the final \circleclar{} estimate
        \(\widehat{\mathcal{M}}_T(x)\).
    }
    \label{fig:consistency-filter}
\end{wrapfigure}
Furthermore, precision and recall breakdowns in Appendix~\ref{app:full-precision-recall} show that \circleclar{}'s gains are most consistent in precision. Moreover, evaluations on unambiguous, single-answer questions in Appendix~\ref{app:single-answer-qa} confirm that our method largely avoids hallucinating non-existent alternatives in unambiguous cases and, in fact, significantly outperforms other methods on smaller models. Finally, Appendix~\ref{app:sensitivity} examines sensitivity to the enumeration and recovery budgets and the recovery threshold. Increasing the enumeration budget improves performance for the GPT models on AmbigQA-HQ but decreases it for the Qwen models; on AmbigQA, performance is largely stable or decreases. This contrast may indicate that smaller models are primarily constrained by precision, whereas larger models are constrained by recall.


\paragraph{Effect of the consistency check.}

To isolate the effect of the consistency check, we compare the unfiltered candidate set $\widetilde{\mathcal{Y}}_T(x)$ which is obtained after the answer enumeration step, with the final estimate $\widehat{\mathcal{M}}_T(x)$ obtained after consistency filtering.
As shown in \Cref{fig:consistency-filter}, the consistency check yields substantial gains in precision across nearly all settings. Because this filtering step can only remove candidates, recall naturally decreases, but the resulting trade-off is generally small compared to the high increase in precision. Notably, the magnitude of the precision gain diminishes with model scale, suggesting that larger models hallucinate fewer invalid initial candidates, and explaining why the simple one-shot enumeration becomes increasingly effective for larger models.

\begin{figure}[t]
\centering

\setlength{\figfourheight}{0.44\textwidth}

\begin{subfigure}[t]{0.38\linewidth}
    \centering
    \resizebox{!}{\figfourheight}{%
        \makebox[327pt][l]{%
            \hspace*{-120pt}%
            \input{figures/heapmaptable.pgf}%
        }%
    }
    \vspace{0.6em}
    \captionsetup{labelformat=empty}
    \caption[short]{\textbf{Figure \thefigure{}a: Ambiguity detection (AUROC).} \circleclar{} is strongest overall; its clearest gains occur on AmbigQA-HQ.}
    \label{fig:qa_ambiguity_detection}
\end{subfigure}\hfill%
\begin{subfigure}[t]{0.58\linewidth}
    \centering
    \raisebox{-0.4em}{%
        \resizebox{!}{\figfourheight}{%
            \input{figures/policy_qa_bag_slim.pgf}%
        }%
    }
    \vspace{0.2em}
    \captionsetup{labelformat=empty}
    \caption[short]{\textbf{Figure \thefigure{}b: Three-way policy (BAcc).} \circleclar{} leads in most settings when deciding whether to abstain, clarify, or answer. Error bars 95\% bootstrap CIs; dashed line is constant-policy.}
    \label{fig:qa_response_policy}
\end{subfigure}
\label{fig:qa_downstream}
\end{figure}

\subsection{Downstream Uses of Model-Realizable Support}
\label{sec:downstream_qa}
As discussed in Section~\ref{sec:motivation}, the same ambiguous input should lead to different behavior from different models, depending on the interpretations they can actually resolve. This motivates using the recovered model-realizable support not only as an object of estimation, but as a basis for deciding whether the model faces unresolved ambiguity and how it should respond. We examine this through model-dependent ambiguity detection and response selection of whether to abstain, answer, or clarify.

\paragraph{Ambiguity detection.}
Unlike conventional ambiguity evaluation, which asks whether the input itself admits multiple valid answers, i.e., whether $|\mathcal{W}_T(x)| \geq 2$, we label an input as ambiguous for model $M$ when $|\mathcal{M}_T(x)| \geq 2$. As such, the cardinality of the estimated model realizable answer support provides a direct ambiguity score, i.e. $|\widehat{\mathcal{M}}_T(x)|$. We compare against ICE~\citep{pmlr-v235-hou24b}, which uses mutual information between clarifications and answers, CLARA-D~\citep{rosseeuw2025clara}, which measures dispersion among generated clarifications, INTENT-SIM~\citep{zhang-choi-2025-clarify}, which estimates entropy over simulated intents, and a modified APA~\citep{kim-etal-2024-aligning} version \(APA^\dagger\), which computes information gain via semantic entropy of the original vs self-disambiguated question. We evaluate all ambiguity scores using AUROC.

\Cref{fig:qa_ambiguity_detection} demonstrates that the cardinality of the recovered candidate set effectively identifies whether model-dependent ambiguity is present. \circleclar{} achieves the highest overall performance, with its most pronounced advantage occurring on AmbigQA-HQ where it consistently outperforms all other methods across all models. This improvement reflects again the higher quality of AmbigQA-HQ's reference sets: annotations in standard AmbigQA are frequently incomplete or noisy, which artificially distorts the true reference cardinality and degrades the quality of the resulting ambiguity labels (Appendix~\ref{app:ambigqa-hq}). Notably, most baseline methods perform close to random guessing. While prior work shows that these baselines perform well when evaluated against the world ambiguity $|\mathcal{W}_T(x)|$, our findings highlight again the crucial distinction: optimizing for $|\mathcal{W}_T(x)|$ does not translate to detecting whether the model itself perceives a question as ambiguous.

\paragraph{Response policy.}
The question of whether to abstain, answer, or clarify follows naturally from the model-realizable support. If the support is empty, the model has no valid answer it can provide (\textsc{Abstain}); if it contains a single answer, there is no alternative for the model to resolve (\textsc{Answer}); and if it contains multiple answers, the appropriate choice depends on the intended interpretation (\textsc{Clarify}). The ground-truth label is hence obtained following this logic on  $\mathcal{M}_T(x)$ while the predictions for \circleclar{} are based on $\widehat{\mathcal{M}}_T(x)$.
To assess whether explicitly estimating $\widehat{\mathcal{M}}_T(x)$ is beneficial, we compare against methods that prompt an action directly. \emph{Direct (Generic)} asks the model to clarify when ambiguity prevents answering, and otherwise to answer or abstain based on its knowledge. \emph{Direct (Model-Dependent)} instead asks it to apply our model-dependent decision rule implicitly without constructing an answer set. Lastly, BAG~\citep{baan2026clarify} bases its decision by examining previously sampled responses. We evaluate all methods using balanced accuracy.

\Cref{fig:qa_response_policy} shows that \circleclar{} performs best overall across models and datasets. Directly providing the model-dependent decision rule offers little consistent improvement over \emph{Direct (Generic)}, while BAG remains close to the constant-policy baseline in most settings. This indicates that the recovered answer support provides a stronger basis for response selection. Appendix~\ref{app:response-policy-criteria} complements this analysis by evaluating the set-estimating methods in Table~\ref{tab:qa_set_recovery} and additionally taking into account whether the produced answers of the sets are correct. We test this across progressively stricter answer-content requirements, with \circleclar{} generally strongest once the correctness of the recovered answers is taken into account.

\section{Discussion}
\label{sec:discussion}

\subsection{Limitations}
\label{subsec:limitation}

\circleclar{} is a precision-oriented estimator whose recall is limited by the initial candidate set: interpretations that the model can answer when explicitly specified may still be unlikely to arise during candidate generation and therefore remain undiscovered. Relatedly, we focus on the \emph{support} of model-realizable answers rather than the relative likelihood of interpretations, treating recovered alternatives equally through a set representation and its cardinality. Further, successful recovery through an answer--clarification--answer cycle provides evidence of model realizability but does not guarantee correctness, as a model may consistently reproduce a factual error; under black-box access, such confidently held errors are generally difficult to distinguish from correct knowledge. Similarly, under our desiderata, clarifications are defined by their answers, and, as such, a valid clarification could theoretically have an unrelated meaning. However, in practice, as we create the clarifications based on their answers and the question, this does not happen. Our evaluation is also limited by the quality and coverage of ambiguity annotations, which may omit valid interpretations or introduce unsupported ones. We mitigate this by constructing the curated AmbigQA-HQ benchmark (Appendix~\ref{app:ambigqa-hq}), but cannot guarantee complete coverage. Finally, \(\mathcal{M}_T(x)\) captures \emph{answerable} model-dependent ambiguity and does not include interpretations that a model can recognize but cannot answer, nor purely linguistic ambiguity it may detect without resolving. Explicitly separating interpretation recognition from answerability is an interesting direction for future work.


\subsection{A Model-Dependent View of Ambiguity}
\label{subsec:closing}

Our results support a view of ambiguity that distinguishes what an input permits from what a particular model can realize. While the world may admit many valid interpretations, only a subset may be accessible within a model's knowledge, and it is this model-realizable ambiguity that determines the alternatives on which the model can meaningfully act. This distinction complements the conventional, world-centric view of ambiguity and provides a more model-specific account of how underspecified inputs are represented and handled. Our clarification framework provides a natural way to characterize this object by establishing a one-to-one correspondence between valid answers and equivalence classes of clarifications. \circleclar{} is one concrete way of exploiting this correspondence under black-box access, but the underlying perspective is more general. We hope this view provides a useful basis for future work on ambiguity-aware language models and their interaction with users.

\section*{Acknowledgements}
The research presented has been performed in the frame of
the RADELN project funded by TUM Georg Nemetschek Institute Artificial Intelligence for the
Built World (GNI). It is further supported by the Bavarian Ministry of Economic Affairs, Regional
Development and Energy with funds from the Hightech Agenda Bayern. This paper is supported by the DAAD programme Konrad Zuse Schools of Excellence in Artificial Intelligence, sponsored by the Federal Ministry of Research, Technology and Space.

\bibliographystyle{plainnat}
\bibliography{references}


\clearpage
\appendix

\section{Proof of Proposition~\ref{prop:clarification-bijection}}
\label{app:proof-bijection}

\clarificationbijection*

\begin{proof}
By D1, every valid clarification $c \in \mathcal{C}_T(x)$ induces a unique answer $y_c$, and by D2,
$y_c \in \mathcal{W}_T(x)$. By D3, all clarifications in the same equivalence class induce the same answer, so $\Phi_x$ is well-defined.

To show injectivity, suppose
\[
    \Phi_x([c]_x) = \Phi_x([c']_x).
\]
Then $y_c = y_{c'}$, and D3 implies $c \sim_x c'$.
Hence $[c]_x = [c']_x$.

To show surjectivity, let $y \in \mathcal{W}_T(x)$.
By assumption, there exists a valid clarification $c \in \mathcal{C}_T(x)$ that induces $y$.
Therefore,
\[
    \Phi_x([c]_x) = y .
\]

Thus, $\Phi_x$ is both injective and surjective, and hence is a bijection.
\end{proof}

\section{Construction of AmbigQA-HQ}
\label{app:ambigqa-hq}

Our evaluation requires a reliable and sufficiently complete reference set of valid answers for each ambiguous question. AmbigQA's disambiguated rewrites and answers \citep{min-etal-2020-ambigqa} provide a natural empirical reference for \(\mathcal{W}_T(x)\), from which we construct each model-specific \(\mathcal{M}_T(x)\). However, the original annotations do not always meet this requirement. Some questions labeled ambiguous support only one answer under their original wording; their additional answers are obtained only by changing the question. For genuinely ambiguous questions, valid readings may be missing, a rewrite may remain unresolved or introduce an unsupported reading, several rewrites may yield the same answer, or an answer may be factually incorrect. These errors change both the membership and the size of the reference answer set. We therefore construct AmbigQA-HQ primarily by selecting genuinely ambiguous questions from AmbigQA and then verifying, repairing, and completing their disambiguated rewrites and answers. Of the 100 questions, 96 were selected from AmbigQA and four were newly constructed to cover additional ambiguity patterns; their answers were verified against authoritative sources.

\paragraph{Question selection.}
For each AmbigQA question, we ask whether its original wording genuinely supports more than one answer. We include it only when \(\lvert\mathcal{W}_T(x)\rvert>1\). We exclude cases in which the alternatives come from rewrites that change what the question asks. \Cref{tab:ambigqa-hq-selection-examples} shows both kinds of cases. The excluded rewrites either split a multi-part request or introduce a restriction absent from the original question.

\begin{table}[ht]
    \centering
    \caption{Examples of question selection.}
    \label{tab:ambigqa-hq-selection-examples}
    \vspace{4pt}
    \small
    \setlength{\tabcolsep}{4pt}
    \begin{tabular}{@{}p{0.24\linewidth}p{0.45\linewidth}p{0.23\linewidth}@{}}
        \toprule
        \textbf{Original question} & \textbf{Original AmbigQA rewrites} &
        \textbf{Decision} \\
        \midrule
        \raggedright ``When did the Ottoman Empire take over Palestine?'' &
        \raggedright ``When did the ottoman empire take over palestine for the
        \textbf{first time}?''
        \newline ``When did the ottoman empire take over palestine for the
        \textbf{second time}?'' &
        \raggedright \textbf{Included:} \emph{take over} can refer to either of
        two historical events. \tabularnewline
        \addlinespace
        \raggedright ``Is red bell pepper a fruit or vegetable?'' &
        \raggedright ``In terms of \textbf{botanicals}, is a red bell pepper a
        fruit or vegetables?''
        \newline ``In \textbf{culinary}, is a red bell pepper a fruit or
        vegetable?'' &
        \raggedright \textbf{Included:} botanical and culinary classifications
        give different valid answers. \tabularnewline
        \midrule
        \raggedright ``When did the Winter Olympic Games begin and where?'' &
        \raggedright ``\textbf{When} did the winter olympic games begin?''
        \newline ``\textbf{Where} did the first winter olympic games begin?'' &
        \raggedright \textbf{Not ambiguous:} the rewrites split one two-part
        request. \tabularnewline
        \addlinespace
        \raggedright ``What is the latest version of iPhone software?'' &
        \raggedright ``What is the latest version of iphone software,
        \textbf{released in 2017}?''
        \newline ``What is the latest version of iphone software,
        \textbf{released in 2016}?''
        \newline ``What is the latest version of iphone software,
        \textbf{released in June 2014}?'' &
        \raggedright \textbf{Not ambiguous:} \emph{latest} has one answer; the
        years change its scope. \tabularnewline
        \addlinespace
        \raggedright ``Who landed the first quad jump in figure skating?'' &
        \raggedright ``Who was the first \textbf{person} to land a quad jump in
        figure skating?''
        \newline ``Who was the first \textbf{woman} to land a quad jump in
        figure skating?''
        \newline ``Who landed the first quad jump in figure skating
        \textbf{ever}?'' &
        \raggedright \textbf{Not ambiguous:} \emph{first} has one answer;
        ``woman'' changes its scope. \tabularnewline
        \addlinespace
        \raggedright ``Who has the highest number of likes on Instagram?'' &
        \raggedright ``Who has the highest number of likes on instagram
        \textbf{in 2017}?''
        \newline ``Who has the highest number of likes on instagram
        \textbf{in 2016}?''
        \newline ``Who has the highest number of likes on instagram
        \textbf{in 2015}?'' &
        \raggedright \textbf{Not ambiguous:} \emph{highest} has one answer;
        the years change its scope. \tabularnewline
        \bottomrule
    \end{tabular}

\end{table}

\paragraph{Curating rewrites and answers.}
For every question, we manually examine its disambiguated rewrites and answers. Each rewrite should isolate one reading of the original question, and its answer should be factually correct. Guided by D1--D3, we remove unsupported rewrites, split rewrites that still allow several answers, and merge those that lead to the same answer. We also add missing plausible readings and make small wording changes where needed. This gives us a more reliable and complete reference answer set for evaluation. \Cref{tab:ambigqa-hq-repaired-examples} shows four representative examples.

\newcommand{\ann}[2]{%
    \textbf{#1}\,$\rightarrow$\,#2%
}

\begin{table}[t]
    \centering
    \caption{Examples of rewrite and answer curation in AmbigQA-HQ.}
    \label{tab:ambigqa-hq-repaired-examples}
    \vspace{3pt}
    \small
    \renewcommand{\arraystretch}{1.05}
    \setlength{\tabcolsep}{0pt}

    \begin{tabularx}{\linewidth}{
        @{}
        >{\RaggedRight\arraybackslash}X
        @{\hspace{1.8em}}
        >{\RaggedRight\arraybackslash}X
        @{}
    }
        \toprule
        \textbf{Original AmbigQA}
        &
        \textbf{Curated AmbigQA-HQ}
        \\
        \midrule

        \multicolumn{2}{@{}l@{}}{
            \textbf{D1 (Uniqueness): Split a multi-answer rewrite}
        }
        \\[-1pt]
        \multicolumn{2}{@{}p{\linewidth}@{}}{
            \emph{Question: ``Who is the girl in ZZ Top's `Legs' video?''}
        }
        \\[3pt]

        \ann{Salesgirl}{Wendy Frazier}
        \par
        \ann{Eliminator girls}{
            \{Dani\`ele Arnaud, Jeana Tomasino, Kymberly Herrin\}
        }
        &
        \ann{Shoe-store clerk}{Wendy Frazier}
        \par
        \ann{Dark-haired girl}{Jeana Tomasino}
        \par
        \ann{Girl wearing black}{Dani\`ele Arnaud}
        \par
        \ann{Girl wearing red}{Kymberly Herrin}
        \tabularnewline

        \midrule

        \multicolumn{2}{@{}l@{}}{
            \textbf{D2 (Faithfulness): Remove an unsupported rewrite}
        }
        \\[-1pt]
        \multicolumn{2}{@{}p{\linewidth}@{}}{
            \emph{Question: ``When did `The Haves and the Have Nots' start?''}
        }
        \\[3pt]

        \ann{TV series}{May 28, 2013}
        \par
        \ann{Stage play}{2011}
        \par
        \ann{Episodes 105--107}{three dates in 2017}
        &
        \ann{TV series}{May 28, 2013}
        \par
        \ann{Stage play}{2011}
        \tabularnewline

        \midrule

        \multicolumn{2}{@{}l@{}}{
            \textbf{D3 (Answer equivalence): Merge duplicate rewrites}
        }
        \\[-1pt]
        \multicolumn{2}{@{}p{\linewidth}@{}}{
            \emph{Question: ``Who made the song `Total Eclipse of the Heart'?''}
        }
        \\[3pt]

        \ann{Wrote/produced}{Jim Steinman (four rewrites)}
        \par
        \ann{Originally recorded}{Bonnie Tyler (two rewrites)}
        &
        \ann{Wrote and produced}{Jim Steinman}
        \par
        \ann{Originally recorded}{Bonnie Tyler}
        \tabularnewline

        \midrule

        \multicolumn{2}{@{}l@{}}{
            \textbf{Factual correction}
        }
        \\[-1pt]
        \multicolumn{2}{@{}p{\linewidth}@{}}{
            \emph{Question: ``When was the first `Texas Chainsaw Massacre'
            movie made?''}
        }
        \\[3pt]

        \ann{Filmed}{July 1974}
        &
        \ann{Filmed}{July--August 1973}
        \\

        \bottomrule
    \end{tabularx}

\end{table}

\FloatBarrier

\section{Additional Answer-Set Recovery Results}
\label{app:additional-set-recovery}

\subsection{Precision and Recall for Model-Realizable Answer Support}
\label{app:full-precision-recall}

\Cref{tab:full-precision-recall} complements the macro set $F_1$ and exact set match results in \Cref{tab:qa_set_recovery} with macro precision and recall. For each input $x$ we compute
\[
P(x)=\frac{\left|\widehat{\mathcal{M}}_T(x)\cap\mathcal{M}_T(x)\right|}{\left|\widehat{\mathcal{M}}_T(x)\right|},
\qquad
R(x)=\frac{\left|\widehat{\mathcal{M}}_T(x)\cap\mathcal{M}_T(x)\right|}{\left|\mathcal{M}_T(x)\right|}.
\]
If the estimated support is empty, precision is one only when the reference support is also empty, and zero otherwise. Recall is one whenever the reference support is empty. All other evaluation choices match \Cref{tab:qa_set_recovery}. Overall, \circleclar{} shows a more consistent advantage in precision than in recall.

\begin{table}[t]
    \centering
    \caption{Macro precision and recall (\%) corresponding to the set-$F_1$ results in \Cref{tab:qa_set_recovery}. Each cell reports $P/R$; recall is one when the reference support is empty. Green and orange mark the best and second-best method, respectively.}
    \label{tab:full-precision-recall}
    \vspace{3pt}
    \setlength{\tabcolsep}{2pt}
    \resizebox{\columnwidth}{!}{%

\definecolor{colfirst}{HTML}{029E73}   
\definecolor{colsecond}{HTML}{DE8F05}  

\newcommand{\prcell}[2]{%
    $#1$\,/\,$#2$%
}

\begin{tabular}{ll|cccc}
\toprule

\multirow{2}{*}{\textbf{Dataset}} &
\multirow{2}{*}{\textbf{Method}} &
\textbf{Qwen3-4B} &
\textbf{Qwen3-30B-A3B} &
\textbf{GPT-5.4-mini} &
\textbf{GPT-5.5} \\

& &
$P$ / $R$ &
$P$ / $R$ &
$P$ / $R$ &
$P$ / $R$ \\
\midrule

\multirow{4}{*}{\shortstack{\small AmbigQA\\\small HQ}} & \small{ICE}
& \prcell{\textcolor{colsecond}{\mathbf{39.83}}}{\textcolor{colsecond}{\mathbf{83.58}}}
& \prcell{\textcolor{colsecond}{\mathbf{54.58}}}{\textcolor{colsecond}{\mathbf{74.67}}}
& \prcell{\textcolor{colsecond}{\mathbf{69.08}}}{\textcolor{colfirst}{\mathbf{75.77}}}
& \prcell{78.50}{\textcolor{colfirst}{\mathbf{84.75}}}
\\

 & \small{Sampling-50}
& \prcell{37.00}{81.58}
& \prcell{54.37}{69.15}
& \prcell{66.55}{67.40}
& \prcell{\textcolor{colfirst}{\mathbf{84.19}}}{58.02}
\\

 & \small{1-Enumeration}
& \prcell{26.34}{\textcolor{colfirst}{\mathbf{84.58}}}
& \prcell{39.92}{\textcolor{colfirst}{\mathbf{79.98}}}
& \prcell{60.25}{74.57}
& \prcell{79.15}{\textcolor{colsecond}{\mathbf{68.62}}}
\\

 & \cellcolor[gray]{0.9}{\small{\circleclar{}}}
& \cellcolor[gray]{0.9}{\prcell{\textcolor{colfirst}{\mathbf{63.50}}}{79.58}}
& \cellcolor[gray]{0.9}{\prcell{\textcolor{colfirst}{\mathbf{72.83}}}{73.72}}
& \cellcolor[gray]{0.9}{\prcell{\textcolor{colfirst}{\mathbf{81.93}}}{\textcolor{colsecond}{\mathbf{75.48}}}}
& \cellcolor[gray]{0.9}{\prcell{\textcolor{colsecond}{\mathbf{80.90}}}{67.58}}
\\

\midrule

\multirow{4}{*}{\shortstack{\small AmbigQA}} & \small{ICE}
& \prcell{\textcolor{colsecond}{\mathbf{33.39}}}{\textcolor{colfirst}{\mathbf{87.51}}}
& \prcell{\textcolor{colsecond}{\mathbf{39.86}}}{\textcolor{colfirst}{\mathbf{84.58}}}
& \prcell{\textcolor{colfirst}{\mathbf{56.78}}}{\textcolor{colfirst}{\mathbf{77.83}}}
& \prcell{\textcolor{colfirst}{\mathbf{61.58}}}{\textcolor{colfirst}{\mathbf{83.11}}}
\\

 & \small{Sampling-50}
& \prcell{25.59}{\textcolor{colsecond}{\mathbf{85.07}}}
& \prcell{35.11}{74.55}
& \prcell{44.82}{69.61}
& \prcell{\textcolor{colsecond}{\mathbf{59.83}}}{52.78}
\\

 & \small{1-Enumeration}
& \prcell{15.56}{84.95}
& \prcell{22.52}{\textcolor{colsecond}{\mathbf{78.77}}}
& \prcell{37.13}{\textcolor{colsecond}{\mathbf{74.81}}}
& \prcell{54.91}{61.14}
\\

 & \cellcolor[gray]{0.9}{\small{\circleclar{}}}
& \cellcolor[gray]{0.9}{\prcell{\textcolor{colfirst}{\mathbf{49.49}}}{81.78}}
& \cellcolor[gray]{0.9}{\prcell{\textcolor{colfirst}{\mathbf{48.35}}}{74.72}}
& \cellcolor[gray]{0.9}{\prcell{\textcolor{colsecond}{\mathbf{50.68}}}{71.13}}
& \cellcolor[gray]{0.9}{\prcell{58.22}{\textcolor{colsecond}{\mathbf{61.21}}}}
\\

\midrule

\multirow{4}{*}{\shortstack{\small AMBROSIA}} & \small{ICE}
& \prcell{\textcolor{colfirst}{\mathbf{53.91}}}{\textcolor{colsecond}{\mathbf{47.22}}}
& \prcell{44.44}{52.78}
& \prcell{50.09}{51.56}
& \prcell{47.40}{\textcolor{colfirst}{\mathbf{62.15}}}
\\

 & \small{Sampling-50}
& \prcell{42.71}{35.42}
& \prcell{\textcolor{colsecond}{\mathbf{52.29}}}{51.56}
& \prcell{43.05}{49.31}
& \prcell{\textcolor{colsecond}{\mathbf{59.15}}}{48.96}
\\

 & \small{1-Enumeration}
& \prcell{43.23}{\textcolor{colfirst}{\mathbf{52.95}}}
& \prcell{48.96}{\textcolor{colfirst}{\mathbf{65.45}}}
& \prcell{\textcolor{colsecond}{\mathbf{56.77}}}{\textcolor{colsecond}{\mathbf{66.32}}}
& \prcell{\textcolor{colfirst}{\mathbf{65.45}}}{\textcolor{colsecond}{\mathbf{55.73}}}
\\

 & \cellcolor[gray]{0.9}{\small{\circleclar{}}}
& \cellcolor[gray]{0.9}{\prcell{\textcolor{colsecond}{\mathbf{51.74}}}{43.06}}
& \cellcolor[gray]{0.9}{\prcell{\textcolor{colfirst}{\mathbf{59.38}}}{\textcolor{colsecond}{\mathbf{60.94}}}}
& \cellcolor[gray]{0.9}{\prcell{\textcolor{colfirst}{\mathbf{60.28}}}{\textcolor{colfirst}{\mathbf{67.53}}}}
& \cellcolor[gray]{0.9}{\prcell{58.37}{53.65}}
\\

\bottomrule
\end{tabular}%
    }

\end{table}

\subsection{Sensitivity Analysis}
\label{app:sensitivity}
We examine the sensitivity of \circleclar{} to the enumeration budget \(B\), the number of recovery samples \(R\), and the recovery threshold \(\tau\). Unless varied, we use the default setting \(B=6\), \(R=10\), and \(\tau=0.6\). The model-specific reference sets are held fixed throughout these analyses.
\paragraph{Enumeration and recovery budgets.}

\Cref{fig:joint-budget-sensitivity} jointly varies the enumeration budget \(B=1,\ldots,6\) and the number of recovery samples \(R=1,\ldots,10\). Increasing \(B\) improves performance for the GPT models on AmbigQA-HQ, while the Qwen models decline as more candidates are introduced. On AmbigQA, performance decreases with larger \(B\) for all models except GPT-5.5, which remains nearly constant. The effect of \(R\) is generally smaller and non-monotonic: larger recovery budgets do not uniformly improve performance.

\begin{figure}[t]
    \centering
    \resizebox{0.98\linewidth}{!}{%
        \input{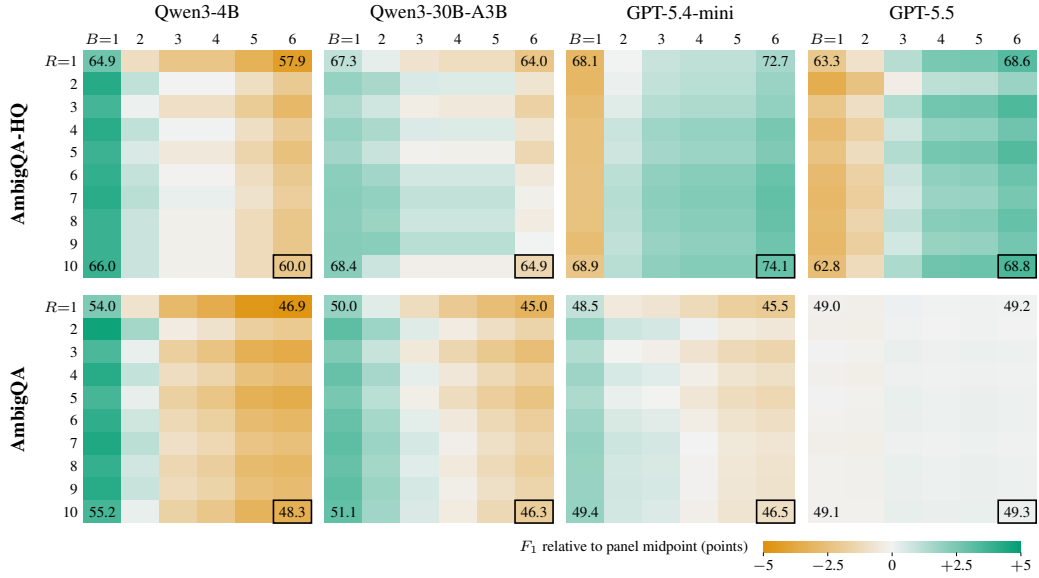}%
    }
    \caption{\textbf{Joint sensitivity to enumeration and recovery budgets.} Macro set \(F_1\) (\%) across enumeration budgets \(B=1,\ldots,6\) and recovery budgets \(R=1,\ldots,10\), with \(\tau=0.6\). Corner labels report absolute \(F_1\), while colors indicate deviations from each panel's midpoint, defined as the mean of its minimum and maximum \(F_1\), using a shared scale in percentage points. Boxed cells mark the default configuration \((B,R)=(6,10)\).}
    \label{fig:joint-budget-sensitivity}
\end{figure}

\paragraph{Recovery threshold.}
\Cref{fig:recovery-threshold} shows the effect of varying \(\tau\). Higher thresholds generally improve the Qwen models, while GPT-5.5 decreases on AmbigQA-HQ. We use \(\tau=0.6\) as the default setting.

\begin{figure}[t]
    \centering
    \resizebox{0.98\linewidth}{!}{
        \input{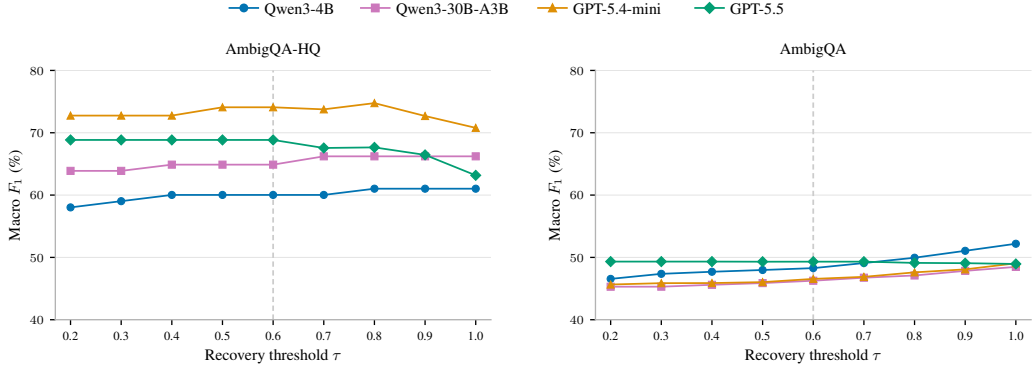}
    }
    \caption{
        Sensitivity to the recovery threshold \(\tau\). The dashed line marks the default value \(\tau=0.6\).
    }
    \label{fig:recovery-threshold}
\end{figure}

\subsection{Answer-Set Recovery on Single-Answer Questions}
\label{app:single-answer-qa}

We additionally evaluate all methods on 100 unambiguous questions from AmbigQA, for which the model-realizable answer support is empty or a singleton. As shown in \Cref{tab:single-answer-qa}, \circleclar{} achieves the highest set F1 for three of the four models. For GPT-5.5, all methods perform similarly, with Sampling-50 slightly
strongest.

\begin{table}[t]
    \centering
    \caption{Macro set $F_1$ and exact set match (EM) (\%) for recovering $\mathcal{M}_T(x)$ on 100 unambiguous AmbigQA questions. Green and orange mark the best and second-best method, respectively.}
    \label{tab:single-answer-qa}
    \vspace{3pt}
    \setlength{\tabcolsep}{2pt}
    \resizebox{\columnwidth}{!}{%

\definecolor{colfirst}{HTML}{029E73}
\definecolor{colsecond}{HTML}{DE8F05}

\newcommand{\singlefem}[2]{%
    $#1$\,/\,$#2$%
}

\begin{tabular}{ll|cccc}
\toprule
\multirow{2}{*}{\textbf{Dataset}} &
\multirow{2}{*}{\textbf{Method}} &
\textbf{Qwen3-4B} &
\textbf{Qwen3-30B-A3B} &
\textbf{GPT-5.4-mini} &
\textbf{GPT-5.5} \\
& &
$F_1$ / EM &
$F_1$ / EM &
$F_1$ / EM &
$F_1$ / EM \\
\midrule

\multirow{4}{*}{\shortstack{\small AmbigQA\\\small Single-answer}} & \small{ICE}
& \singlefem{43.90}{37.00}
& \singlefem{61.32}{54.00}
& \singlefem{71.20}{59.00}
& \singlefem{\textcolor{colsecond}{\mathbf{88.50}}}{\textcolor{colfirst}{\mathbf{86.00}}}
\\

 & \small{Sampling-50}
& \singlefem{\textcolor{colsecond}{\mathbf{54.22}}}{\textcolor{colsecond}{\mathbf{53.00}}}
& \singlefem{\textcolor{colsecond}{\mathbf{69.17}}}{\textcolor{colfirst}{\mathbf{68.00}}}
& \singlefem{\textcolor{colsecond}{\mathbf{78.83}}}{\textcolor{colsecond}{\mathbf{77.00}}}
& \singlefem{\textcolor{colfirst}{\mathbf{88.67}}}{\textcolor{colfirst}{\mathbf{86.00}}}
\\

 & \small{1-Enumeration}
& \singlefem{48.67}{48.00}
& \singlefem{57.40}{40.00}
& \singlefem{71.83}{59.00}
& \singlefem{87.67}{\textcolor{colsecond}{\mathbf{84.00}}}
\\

 & \cellcolor[gray]{0.9}{\small{\circleclar{}}}
& \cellcolor[gray]{0.9}{\singlefem{\textcolor{colfirst}{\mathbf{68.67}}}{\textcolor{colfirst}{\mathbf{68.00}}}}
& \cellcolor[gray]{0.9}{\singlefem{\textcolor{colfirst}{\mathbf{70.07}}}{\textcolor{colsecond}{\mathbf{67.00}}}}
& \cellcolor[gray]{0.9}{\singlefem{\textcolor{colfirst}{\mathbf{82.00}}}{\textcolor{colfirst}{\mathbf{80.00}}}}
& \cellcolor[gray]{0.9}{\singlefem{87.17}{\textcolor{colsecond}{\mathbf{84.00}}}}
\\

\bottomrule
\end{tabular}%
    }
\end{table}

\section{Content-Aware Evaluation of Set-Induced Response Policies}
\label{app:response-policy-criteria}

For an input $x$, let $\mathcal{M}=\mathcal{M}_T(x)$ be the model-realizable reference answer support and let $\widehat{\mathcal{M}}=\widehat{\mathcal{M}}_T(x)$ be a method's estimated answer support. The reference support induces \textsc{Abstain}, \textsc{Answer}, or \textsc{Clarify} when $|\mathcal{M}|$ is zero, one, or at least two, respectively. For any answer support \(S\), let \(a(S)\) denote the induced response action: \textsc{Abstain} if \(|S|=0\), \textsc{Answer} if \(|S|=1\), and \textsc{Clarify} if \(|S|\geq 2\).

\paragraph{Criteria.} 
The main policy experiment in Section~5.2 evaluates whether the estimated and reference supports induce the same response action:

\[
S_{\mathrm{action}}(\widehat{\mathcal{M}},\mathcal{M})
=
\mathbbm{1}\!\left[
a(\widehat{\mathcal{M}})=a(\mathcal{M})
\right].
\]

This action-only criterion permits a common comparison with Direct (Generic) and Direct (Model-Dependent), which predict an action without constructing an answer set. Here, we additionally evaluate whether methods that estimate an answer set recover content consistent with the reference support. We include the action-only criterion as a reference point and progressively strengthen the answer-content requirement. We rescore the same fixed estimates used in Section~\ref{sec:set-recovery}. Direct policies are omitted because they do not output answer sets.

The first content-aware criterion requires the correct unique answer in the \textsc{Answer} regime and at least two matched reference answers in the \textsc{Clarify} regime:

\[
S_{\mathrm{two}}(\widehat{\mathcal{M}},\mathcal{M})=
\begin{cases}
\mathbbm{1}\!\left[\widehat{\mathcal{M}}=\varnothing\right], & |\mathcal{M}|=0,\\
\mathbbm{1}\!\left[\widehat{\mathcal{M}}=\mathcal{M}\right], & |\mathcal{M}|=1,\\
\mathbbm{1}\!\left[|\widehat{\mathcal{M}}\cap\mathcal{M}|\geq2\right], & |\mathcal{M}|\geq2.
\end{cases}
\]

For \textsc{Clarify}, this criterion permits additional predictions outside the reference support. A stricter supported-subset criterion additionally requires every predicted answer to be reference-supported:

\[
S_{\mathrm{subset}}(\widehat{\mathcal{M}},\mathcal{M})=
\begin{cases}
\mathbbm{1}\!\left[\widehat{\mathcal{M}}=\varnothing\right], & |\mathcal{M}|=0,\\
\mathbbm{1}\!\left[\widehat{\mathcal{M}}=\mathcal{M}\right], & |\mathcal{M}|=1,\\
\mathbbm{1}\!\left[|\widehat{\mathcal{M}}|\geq2 \,\land\, \widehat{\mathcal{M}}\subseteq\mathcal{M}\right], & |\mathcal{M}|\geq2.
\end{cases}
\]

Finally, exact-set success requires complete recovery of the reference support:

\[
S_{\mathrm{exact}}(\widehat{\mathcal{M}},\mathcal{M})=\mathbbm{1}\!\left[\widehat{\mathcal{M}}=\mathcal{M}\right].
\]

The four criteria form the following ordering in strictness:

\[
S_{\mathrm{exact}}=1
\;\Longrightarrow\;
S_{\mathrm{subset}}=1
\;\Longrightarrow\;
S_{\mathrm{two}}=1
\;\Longrightarrow\;
S_{\mathrm{action}}=1.
\]

For each criterion, we compute the success rate separately within the three reference-cardinality regimes and report their unweighted mean, mirroring the balanced treatment of the three response actions in the main experiment.

\begin{figure}[t!]
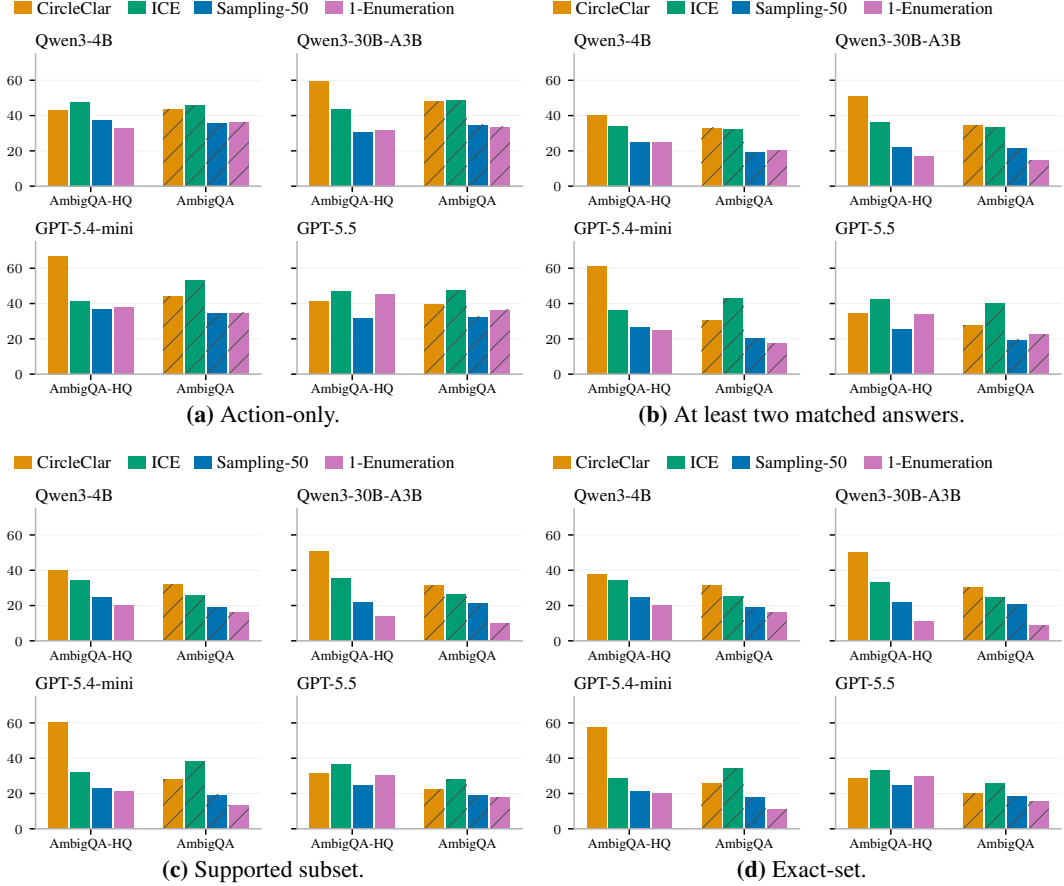

\centering
\begin{subfigure}[t]{0.49\textwidth}
\centering
\resizebox{\linewidth}{!}{\input{figures/appendix_figures/policy_action_only.pgf}}
\caption{Action-only.}
\label{fig:policy-action-only}
\end{subfigure}\hfill
\begin{subfigure}[t]{0.49\textwidth}
\centering
\resizebox{\linewidth}{!}{\input{figures/appendix_figures/policy_relaxed_answer_coverage.pgf}}
\caption{At least two matched answers.}
\label{fig:policy-two-matched}
\end{subfigure}

\medskip

\begin{subfigure}[t]{0.49\textwidth}
\centering
\resizebox{\linewidth}{!}{\input{figures/appendix_figures/policy.pgf}}
\caption{Supported subset.}
\label{fig:policy-main}
\end{subfigure}\hfill
\begin{subfigure}[t]{0.49\textwidth}
\centering
\resizebox{\linewidth}{!}{\input{figures/appendix_figures/policy_exact_set.pgf}}
\caption{Exact-set.}
\label{fig:policy-exact-set}
\end{subfigure}

\caption{Content-aware response-policy evaluation of the four methods. Each panel reports the unweighted mean of the success rates in the \textsc{Abstain}, \textsc{Answer}, and \textsc{Clarify} reference-cardinality regimes. Higher is better.}
\label{fig:response-policy-criterion-sensitivity}
\end{figure}

Figure~\ref{fig:response-policy-criterion-sensitivity} shows that action-only rankings are mixed. Once answer identity is taken into account, \textsc{CircleClar} is generally strongest on AmbigQA-HQ, whereas results on the original AmbigQA are more mixed. The difference between the action-only and content-aware evaluations shows that predicting the correct response mode does not necessarily imply recovering the correct answers: a method may estimate the appropriate support cardinality while assigning that cardinality to unsupported content.

\FloatBarrier

\section{Experimental Settings and Prompts}
\label{app:implementation}

\paragraph{Overview.}
Across QA and text-to-SQL, we collapse identical outputs and use GPT-5.4-mini at temperature 0 for semantic grouping and equivalence judgment. Before scoring, we semantically group each method's predictions under the original input and match them with the model-specific reference answer groups. We compute precision, recall, set $F_1$, and exact set match (EM) from these matches. Within each task, the same equivalence judge is used during reference construction and the \circleclar{} consistency check. Task-specific generation and set-construction rules and prompts are given below. We use a fixed recovery threshold $\tau=0.6$ for \circleclar{} across all models and datasets. Under our default recovery budgets, this corresponds to requiring a strict majority of re-answers to fall in the semantic answer group matching the original candidate. For the construction of  $\mathcal{M}_T(x)$, we use the same recovery criterion as \circleclar{}, using stricter thresholds of $0.9$ for QA and $0.8$ for Text-to-SQL.

\subsection{Open-Domain Question Answering}
 Let $m$ be the mass of the top re-answer group. \Cref{tab:qa-frozen-configuration} summarizes the QA generation budgets and set-construction rules.

\begin{table}[t]
    \centering
    \caption{QA generation budgets and set-construction rules.}
    \label{tab:qa-frozen-configuration}
    \vspace{3pt}
    \small
    \setlength{\tabcolsep}{3pt}
    \renewcommand{\arraystretch}{1.10}
    \begin{tabularx}{\columnwidth}{@{}p{0.24\columnwidth}p{0.39\columnwidth}X@{}}
        \toprule
        \textbf{Method} & \textbf{Generation} & \textbf{Set construction} \\
        \midrule
        Reference construction
        & 10 answers per dataset-provided rewrite, $T=0.3$.
        & Accept if $m\geq0.9$, and the top answer matches the candidate. \\
        \midrule

        \circleclar{}
        & 6 answer-set enumerations, $T=0.5$; one clarification per candidate,
          $T=0$; 10 re-answers per clarification, $T=0.3$.
        & Apply the same consistency check as in reference construction to each
          candidate, with $m\geq0.6$. \\

        ICE
        & 2 clarification samples, $T=1.0$, and 8 answers per clarification, $T=0.5$,
          using the released AmbigQA prompts.
        & Threshold answer counts separately within each clarification, union the groups, and sweep $t=1,\ldots,8$. \\

        Sampling-50
        & 50 single-answer samples from the unresolved question, $T=0.5$.
        & Discard \texttt{UNKNOWN} outputs, use the remaining sampled answers directly. \\

        1-Enumeration
        & One answer-set enumeration, $T=0.5$, using the \circleclar{} prompt.
        & Use the enumerated answers directly. \\
        \bottomrule
    \end{tabularx}
\end{table}

\subsubsection{Prompt Templates}
ICE uses the released AmbigQA few-shot prompts of \citet{pmlr-v235-hou24b} without modification, so we do not duplicate them here.

\newtcblisting{promptbox}[1]{
    enhanced,
    breakable,
    listing only,
    listing engine=listings,
    title={#1},
    colback=black!2,
    colframe=black!35,
    boxrule=0.4pt,
    arc=1pt,
    left=4pt,
    right=4pt,
    top=3pt,
    bottom=3pt,
    fonttitle=\small\bfseries,
    listing options={
        basicstyle=\ttfamily\scriptsize,
        breaklines=true,
        columns=fullflexible,
        keepspaces=true,
        showstringspaces=false
    }
}

\begin{promptbox}{Answer-set enumeration (\circleclar{} and 1-Enumeration)}
SYSTEM:
You answer open-domain questions by listing candidate final answers.

List all distinct final answers that could be correct under the original question or different reasonable readings of the original question.

Each item must be a standalone answer span, not a sentence or explanation.
Use the full canonical surface form for each answer.
For people, output full names, not initials, surnames, or shortened names.
For organizations, works, places, and events, output the full official or commonly used name, not an acronym or shorthand.
For full dates, use full month names. For counts or measurements, prefer digits.

Do not ask clarifying questions. Do not explain. Do not output aliases or formatting variants as separate answers.

Output exactly a numbered list:
1. answer
2. answer

USER:
Question: {question}

Answers:
\end{promptbox}

\begin{promptbox}{Semantic grouping of predicted answers}
SYSTEM:
You are a semantic clustering engine.

USER:
Task:
Cluster the answers by their semantic meaning for the given question.

First, silently identify a few semantic tags for each answer.
Then cluster answers that have the same core meaning.

Rules:
- Do not judge factual correctness.
- Do not add missing answers.
- Each answer must appear in exactly one cluster.
- Use as many or as few clusters as needed.
- Answers should be in the same cluster only if they would be treated as the same answer to the question.
- If two answers differ in date, number, or specificity, keep them separate.

Return JSON ONLY.
{
  "clusters": [
    {
      "cluster_id": 0,
      "member_indices": [0, 2],
      "cluster_tags": ["short tag"]
    }
  ]
}

Question:
{question}

Answers:
{numbered_answers}
\end{promptbox}

\begin{promptbox}{Answer-conditioned clarification}
SYSTEM:
You rewrite ambiguous open-domain questions into unambiguous questions.

Given an original question and a target answer, write one clarified question whose unique answer is the target answer under some reasonable interpretation of the original question.

Do not fact-check beyond choosing a reasonable interpretation. Do not explain. Do not include the target answer itself in the clarified question unless it already appears in the original question. Do not change the kind of answer the original question asks for.

The clarified question must not be identical to the original question.
The clarified question should add enough constraint that the target answer, rather than a generic or more common answer, would be the unique answer.

Output exactly one line:
Clarification: clarified question?

USER:
Original question: {question}
Target answer: {source_answer}

Clarification:
\end{promptbox}

\begin{promptbox}{Resolved-question answering (reference construction and \circleclar{})}
SYSTEM:
You answer fact-seeking questions directly.

Return exactly one standalone final answer span, not a sentence, list, explanation, citation, or hedge.

Use the full canonical surface form for the answer.
For people, output full names, not initials, surnames, or shortened names.
For organizations, works, places, and events, output the full official or commonly used name, not an acronym or shorthand.
For full dates, use full month names. For counts or measurements, prefer digits.

If several aliases would refer to the same answer, choose one canonical form.
Do not output aliases or formatting variants as separate answers.

USER:
Question: {resolved_question}

Answer:
\end{promptbox}

\begin{promptbox}{Sampling-50 direct answering}
SYSTEM:
You answer open-domain factual questions directly.

Return exactly one short factual answer to the user's original question.
Do not list alternatives, ask a clarification question, explain your reasoning, or add a label.
Prefer a short canonical answer string such as a person, date, place, organization, title, number, or noun phrase.
If you cannot provide a plausible factual answer from your own knowledge, return exactly UNKNOWN.

Output exactly one non-empty line.

USER:
Question: {question}

Answer:
\end{promptbox}

\begin{promptbox}{Prediction-group to reference-group judgment}
SYSTEM:
You are an answer-equivalence judge. Do not answer the question yourself.

USER:
Question:
{question}

Canonical answer groups:
{canonical_groups_json}

Predicted semantic cluster:
{target_prediction_json}

The strings in the predicted cluster have already been grouped as one answer.
Decide whether that answer is equivalent to exactly one canonical group.
Treat the canonical groups as authoritative. Return NONE if it matches no group,
could match more than one group, contradicts a group, or is unclear.

Return JSON only, with `group_id` set to one supplied group ID or `NONE`:
{"group_id": "G1", "reason": "one concise sentence"}
\end{promptbox}

\begin{promptbox}{Pairwise answer-equivalence judgment (reference construction and \circleclar{} consistency check)}
SYSTEM:
You are an expert semantic equivalence judge. Do not answer the question.

USER:
Question: {question}
Answer A: {ref_text}
Answer B: {candidate}

Do Answer A and Answer B convey the same essential answer to the question?

Return true only if they express the same core fact, entity, or value, regardless of phrasing or formatting.
Aliases, former names, stage names, pen names, nicknames, and character identities count as equivalent when they identify the same entity.
Return false if they differ, contradict each other, are merely related, or equivalence is unclear.

Respond as JSON:
{"score": true or false, "reason": "one sentence"}
\end{promptbox}

\subsection{Executable Text-to-SQL}

 Let $m$ be the mass of the top re-answer group. \Cref{tab:sql-frozen-configuration} summarizes the SQL generation budgets and set-construction rules.

\begin{table}[H]
    \centering
    \caption{Text-to-SQL generation budgets and set-construction rules.}
    \label{tab:sql-frozen-configuration}
    \vspace{3pt}
    \small
    \setlength{\tabcolsep}{3pt}
    \renewcommand{\arraystretch}{1.10}
    \begin{tabularx}{\columnwidth}{@{}p{0.24\columnwidth}p{0.39\columnwidth}X@{}}
        \toprule
        \textbf{Method} & \textbf{Generation} & \textbf{Set construction} \\
        \midrule
        Reference construction
        & 5 SQL samples per dataset-provided resolved interpretation,
          $T=0.2$.
        & Accept if $m\geq0.8$ and the top answer matches the candidate. \\
        \midrule

        \circleclar{}
        & 5 answer-set enumerations, $T=1.0$; one clarification per grouped candidate, $T=0$; 5 recovery queries per clarification, $T=0.2$.
        & Execute and semantically group the candidates; apply the same consistency check with $m\geq0.6$. \\

        ICE
        & 2 clarification samples, $T=1.0$, using up to 8 distinct rewrites (or the original request if none parse); 8 SQL samples per rewrite, $T=0.5$.
        & Threshold answer counts separately within each clarification, union the groups, and sweep $t=1,\ldots,8$. \\

        Sampling-50
        & 50 single-query samples from the unresolved request, $T=0.5$.
        & Use the sampled answers directly. \\

        1-Enumeration
        & One answer-set enumeration, $T=1.0$, using the \circleclar{} prompt.
        & Use the enumerated answers directly. \\
        \bottomrule
    \end{tabularx}

\end{table}

\subsubsection{Prompt Templates}

\begin{promptbox}{SQL answer-set enumeration (\circleclar{} and 1-Enumeration)}
The request might be ambiguous. These ambiguities can arise from various factors, including but not limited to:

1. Ambiguous references to entities in the request.
2. Multiple properties of objects or entities leading to different interpretations.
3. Unclear timestamps or time conventions.
4. Unclear locations or spatial scope.
5. Multiple valid answer types based on the request.

The task is to write SQLite queries based on the provided request. The request can be ambiguous and admit different reasonable interpretations. Write all SQL queries corresponding to the reasonable interpretations you can identify, and separate each SQL query with an empty line.

Do not include explanations, labels, Markdown fences, or duplicate/formatting variants. Do not select extra columns beyond those requested. Each query must be a single read-only SELECT or WITH statement.

Given the following SQLite database schema and content:

DATABASE_DUMP

Answer the following unresolved request:
UNRESOLVED_REQUEST

Return only the SQL queries, separated by empty lines.
\end{promptbox}

\begin{promptbox}{Semantic grouping of SQL answers}
Partition the SQL outputs by the reasonable interpretation of the input request that they answer.

Compare meaning, not table shape. Merge outputs with the same requested entities, filters, scope, grouping or aggregation, time convention, alternatives, and facts even when one adds ancillary columns, reorders columns, changes SQL form, or uses a lossless vertical/horizontal layout.

Critical contrast: if an ambiguous phrase has plausible facts A and B, keep A-only, B-only, and A+B as separate interpretations. An added field is ancillary only when it is context; it is not ancillary when it realizes another live reading. Sharing a column or having a lossless projection is therefore evidence, not the decision.

Use the input, schema, SQL, and complete typed results together. Assign each output to exactly one cluster. If semantic equivalence remains uncertain, keep the outputs separate.

Return only JSON:
{"clusters":[{"members":["id1","id2"],"interpretation":"short natural-language description"}]}

Input request:
{{INPUT_REQUEST}}

Database schema:
{{DATABASE_SCHEMA}}

SQL and results:
{{SQL_RESULTS}}
\end{promptbox}

\begin{promptbox}{Answer-conditioned SQL clarification (\circleclar{})}
SYSTEM:
You rewrite ambiguous database requests into unambiguous requests.

Given an original request and a target query, write one clarification that expresses the query's interpretation as a reasonable reading and rules out the other reasonable readings.

Do not explain or mention SQL, schema names, or result values. Do not change what kind of output the original request asks for.

The clarification must not be identical to the original or merely cosmetic.
Never output REJECT.

Output exactly one line:
Clarification: clarified request

USER:
SQLite database context:

DATABASE_DUMP

Original request:
UNRESOLVED_REQUEST

Target query:
REPRESENTATIVE_SQL

Clarification:
\end{promptbox}

\begin{promptbox}{Resolved-request SQL generation (reference construction and \circleclar{})}
The task is to write exactly one SQLite query for the resolved request below. The request has already been disambiguated. Do not include explanations, Markdown fences, or alternative queries, and do not select extra columns beyond those requested.

Given the following SQLite database schema and content:

DATABASE_DUMP

Resolved request:
RESOLVED_REQUEST

Return exactly one SQL query.
\end{promptbox}

\begin{promptbox}{Direct SQL generation (Sampling-50 and ICE answering)}
You are given a SQLite database dump and one database request.

Database dump:
DATABASE_DUMP

Request:
REQUEST

Write exactly one read-only SQLite SELECT or WITH query that answers the request. 
Return only the SQL query, with no prose, label, or Markdown fence.
\end{promptbox}

\begin{promptbox}{ICE clarification generation}
The request might be ambiguous. These ambiguities can arise from various
factors, including but not limited to:

1. Ambiguous references to entities in the request.
2. Multiple properties of objects or entities leading to different
   interpretations.
3. Unclear timestamps or time conventions.
4. Unclear locations or spatial scope.
5. Multiple valid answer types based on the request.

You are given one unresolved natural-language database request and the complete
SQLite database dump available to the answering model.

Database dump:
DATABASE_DUMP

Unresolved request:
UNRESOLVED_REQUEST

Provide one distinct rephrasing for each reasonable reading of the unresolved
request that you can identify. By "rephrasing," we mean reformulating the
request to be clear and direct, eliminating ambiguity without altering its
original intent. Each rephrasing must select exactly one reasonable reading,
preserve the requested output, and be answerable from this database.

Do not seek further information, write SQL, answer the request, mention
ambiguity labels, or add a reading unsupported by the wording and database.
Return at most eight rephrasings. If the original request is already clear and
unambiguous, return an empty list.

Return exactly one JSON object with this schema and no other text:
{"clarifications":["first rewritten request","second rewritten request"]}
\end{promptbox}

\begin{promptbox}{SQL prediction-group to reference-group judgment}
Judge whether one predicted SQL answer group is equivalent to exactly one reference answer group. Do not answer the request or write SQL.

The predicted group has already been semantically grouped and should be judged as one answer. Use the request, schema, SQL, and complete typed results. Compare meaning, not SQL syntax or table shape. Allow ancillary columns, column reordering, and lossless reshaping only when the requested entities, filters, scope, grouping or aggregation, time convention, alternatives, and facts are preserved. Do not accept accidental result equality when the SQL meanings differ.

Treat the reference groups as authoritative. Return NONE if the prediction matches no group, could match more than one group, changes a requested fact or interpretation, or is unclear.

Return only JSON, with group_id set to one supplied group ID or NONE:
{"group_id":"G1","reason":"one concise sentence"}

Input request:
{{INPUT_REQUEST}}

Database schema:
{{DATABASE_SCHEMA}}

Reference answer groups:
{{REFERENCE_GROUPS_JSON}}

Predicted answer group:
{{PREDICTED_GROUP_JSON}}
\end{promptbox}

\begin{promptbox}{Consistency-check judgment (reference construction and \circleclar{})}
Judge whether re-answering the resolved request preserves the source answer on this fixed database.

The source SQL and result are authoritative; never correct or replace them. Use the resolved request, schema, SQL, complete typed results, and deterministic relation checks. Treat the relation checks as evidence, not as the final semantic decision.

Return equivalent only when every requested source fact is preserved. Allow ancillary-column differences and lossless reshaping or derivation. Return not_equivalent if any requested source fact, row, entity, filter, scope, grouping or aggregation, time convention, or live alternative is missing, replaced, or changed. Do not accept accidental result equality when the SQL meanings differ. Return uncertain only when the evidence is insufficient.

Return only JSON:
{"decision":"equivalent|not_equivalent|uncertain","reason":"short reason"}

Resolved input request:
{{RESOLVED_REQUEST}}

Database schema:
{{DATABASE_SCHEMA}}

Deterministic relation checks:
{{DETERMINISTIC_RELATION_CHECK}}

Source SQL and result:
{{SOURCE_SQL_RESULT}}

Stable top-mass recovery SQL and results:
{{STABLE_RECOVERY_SQL_RESULTS}}
\end{promptbox}

\subsection{Direct Response-Policy Baselines}
\label{app:direct-policy-baselines}
Both direct baselines receive only the unresolved question, select a response
action without constructing an answer set and are queried once per input at temperature \(T=0\). We parse the predicted action from the \texttt{action} field.

\begin{promptbox}{Direct (Generic) response-policy baseline}
SYSTEM:
You are an AI assistant choosing how to respond to a user's factual question. Do not answer the question itself and do not write a clarification question. Return only the requested JSON object, with no Markdown or extra text.

USER:
Choose exactly one conversational action for the user's question: ANSWER, CLARIFY, or ABSTAIN.

Use these conversational criteria:
- ANSWER when the question is sufficiently clear and you know the answer.
- CLARIFY when ambiguity or missing details prevent a direct answer, and the user can resolve the issue by specifying their intent without needing to know the factual answer themselves.
- ABSTAIN when the question is sufficiently clear but you do not know the answer.

Choose the action you would take now from the question alone. Do not assume access to search, tools, reference answers, or additional evidence.

Return exactly this JSON schema:
{"action":"ANSWER|CLARIFY|ABSTAIN"}

User question: {question}
\end{promptbox}

\begin{promptbox}{Direct (Model-Dependent) response-policy baseline}
SYSTEM:
You are an AI assistant deciding how to respond to a user's factual question.

Do not answer the question itself.
Do not write a clarification question.
Return only the requested JSON object, with no Markdown or additional text.

USER:
Choose exactly one action for the user's question:
ANSWER, CLARIFY, or ABSTAIN.

Base your decision only on the question and your own current knowledge.

Use the following criteria:

- ANSWER if you can reliably provide exactly one distinct answer. Choose
  ANSWER even if the question has other plausible readings, as long as you
  cannot reliably provide a different answer for any of them.

- CLARIFY if at least two plausible readings of the question lead to different
  answers that you can reliably provide, so the user's intended reading is
  needed before choosing an answer.

- ABSTAIN if you cannot reliably provide an answer under any plausible reading
  of the question.

Do not choose CLARIFY merely because the wording allows multiple readings.
The decision should depend on the different answers you can reliably provide,
not only on ambiguity in the wording.

Do not assume access to search, tools, reference answers, or additional
evidence.

Return exactly:
{"action":"ANSWER|CLARIFY|ABSTAIN"}

User question: {question}
\end{promptbox}


\newpage

\end{document}